%% file: main.tex
\documentclass[10pt,twocolumn,letterpaper]{article}

\usepackage{iccv}
\usepackage{times}
\usepackage{epsfig}
\usepackage{graphicx}
\usepackage{amsmath}
\usepackage{amssymb}
\usepackage[utf8]{inputenc}
\usepackage[english]{babel}
\usepackage{extarrows}
\usepackage{algorithm}
\usepackage{algpseudocode}
\makeatletter
\@namedef{ver@everyshi.sty}{}
\makeatother

\usepackage[page]{appendix}
\usepackage{bm}
\usepackage{caption}
\usepackage{subcaption}
\usepackage{todonotes}
\usepackage{tcolorbox}
\usepackage{pgfplots}
\usepackage{tikz}
\usepackage{tikz-3dplot}
\usepackage{tikzscale}
\usetikzlibrary{arrows,shapes,chains,matrix,positioning}
\usetikzlibrary{scopes,decorations,shadows,backgrounds,fit}
\usetikzlibrary{decorations.pathreplacing,calc,3d,quotes}
\usepackage{pgfmath}
\usepackage{threeparttable}

\usepackage{amsthm}
\newtheorem{theorem}{Theorem}

\newtheorem{lemma}[theorem]{Lemma}

\usepackage{array}
\usepackage{mdwmath}
\usepackage{mdwtab}
\usepackage{eqparbox}
\usepackage{fixltx2e}
\usepackage{stfloats}
\usepackage{url}
\usepackage{diagbox}
\usepackage{verbatim}
\usepackage{booktabs}
\usepackage{multirow}
\usepackage{mathtools}
\usepackage{breqn}

\usepackage{filecontents}
\usepackage{pgfplots, pgfplotstable}

\newcommand{\M}{\mathcal M}
\newcommand{\N}{\mathcal N}

\newcommand{\G}{\mathcal G}
\newcommand{\R}{\mathbb {R}}

\newcommand{\E}{\mathbb {E}}

\newcommand{\nn}{\nonumber}

\newcommand{\mse}{\mathrm{MSE}}

\newcommand{\tr}{\mathrm{Tr}}

\newcommand{\one}{{\bf{1}}}

\newcommand{\J}{J}
\newcommand{\x}{{\bf x}}
\newcommand{\z}{{\bf z}}
\newcommand{\s}{{\bf s}}
\newcommand{\y}{{\bf y}}
\newcommand{\yh}{{\bf \hat y}}

\newcommand{\sh}{{\bf \hat s}}
\newcommand{\bb}{{\bf b}}
\newcommand{\bc}{{\bf c}}
\newcommand{\bM}{{\bf M}}
\newcommand{\bK}{{\bf K}}
\newcommand{\bt}{{\bf \Theta}}

\newcommand{\bW}{{\bf W}}
\newcommand{\bI}{{\bf I}}
\newcommand{\bC}{{\bf C}}
\newcommand{\bS}{{\bf S}}
\newcommand{\bY}{{\bf Y}}
\newcommand{\bX}{{\bf X}}
\newcommand{\bL}{{\bf L}}
\newcommand{\bG}{{\bf G}}

\newcommand{\bD}{{\bf D}}

\newcommand{\bV}{{\bf V}}
\newcommand{\bB}{{\bf B}}
\newcommand{\bQ}{{\bf Q}}
\newcommand{\bp}{{\bf \Phi}}
\newcommand{\bl}{{\bf \Lambda}}
\usepackage[pagebackref=true,breaklinks=true,letterpaper=true,colorlinks,bookmarks=false]{hyperref}

\iccvfinalcopy % *** Uncomment this line for the final submission

\def\iccvPaperID{5016} % *** Enter the ICCV Paper ID here

\begin{document}

%%%%%%%%% TITLE
\title{On the Global Optima of Kernelized Adversarial Representation Learning}
\author{Bashir Sadeghi\\
Michigan State University\\
{\tt\small sadeghib@msu.edu}
\and
Runyi Yu\\
Eastern Mediterranean University \\
{\tt\small yu@ieee.org}
\and
Vishnu Naresh Boddeti\\
Michigan State University\\
{\tt\small vishnu@msu.edu}
}

\maketitle
\thispagestyle{empty}

\input{abstract.tex}

\input{introduction.tex}

\input{related-work.tex}

\input{approach.tex}

\input{experiments.tex}

\input{conclusion.tex}

{\small
 \bibliographystyle{ieee_fullname}
\bibliography{mybib}
}

\onecolumn
\begin{appendices}
	\input{appendix.tex}
\end{appendices}

\end{document}

%% file: abstract.tex
% abstract.tex

\begin{abstract}
Adversarial representation learning is a promising paradigm for obtaining data representations that are invariant to certain sensitive attributes while retaining the information necessary for predicting target attributes. Existing approaches solve this problem through iterative adversarial minimax optimization and lack theoretical guarantees. In this paper, we first study the ``linear" form of this problem i.e., the setting where all the players are linear functions. We show that the resulting optimization problem is both non-convex and non-differentiable. We obtain an exact closed-form expression for its global optima through spectral learning and provide performance guarantees in terms of analytical bounds on the achievable utility and invariance. We then extend this solution and analysis to non-linear functions through kernel representation. Numerical experiments on UCI, Extended Yale B and CIFAR-100 datasets indicate that, (a) practically, our solution is ideal for ``imparting" provable invariance to any biased pre-trained data representation, and (b) empirically, the trade-off between utility and invariance provided by our solution is comparable to iterative minimax optimization of existing deep neural network based approaches. Code is available at \url{https://github.com/human-analysis/Kernel-ARL}
\end{abstract}

%% file: introduction.tex
% introduction.tex

\section{Introduction}

Adversarial representation learning (ARL) is a promising framework for training image representation models that can control the information encapsulated within it. ARL is practically employed to learn representations for a variety of applications, including, unsupervised domain adaptation of images \cite{ganin2015unsupervised}, censoring sensitive information from images \cite{edwards2015censoring}, learning fair and unbiased representations \cite{louizos2015variational,madras2018learning}, learning representations that are controllably invariant to sensitive attributes \cite{xie2017controllable} and mitigating unintended information leakage \cite{roy2019mitigating}, amongst others.
\begin{figure}[t]
    \centering
    \includegraphics[width=0.4\textwidth]{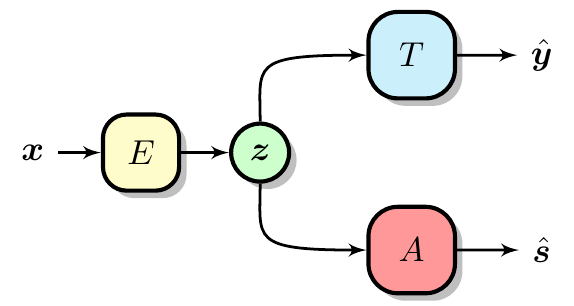}
    \caption{\textbf{Adversarial Representation Learning} consists of three entities, an encoder $E$ that obtains a compact representation $\bm{z}$ of input data $\bm{x}$, a predictor $T$ that predicts a desired target attribute $\bm{y}$ and an adversary that seeks to extract a sensitive attribute $\bm{s}$, both from the embedding $\bm{z}$.\label{fig:arl}}
\end{figure}

At the core of the ARL formulation is the idea of jointly optimizing three entities: (i) An encoder $E$ that seeks to distill the information from input data and retains the information relevant to a target task while \emph{intentionally} and \emph{permanently} eliminating the information corresponding to a sensitive attribute, (ii) A predictor $T$ that seeks to extract a desired target attribute, and (iii) A proxy adversary $A$, playing the role of an unknown adversary, that seeks to extract a known sensitive attribute. Figure \ref{fig:arl} shows a pictorial illustration of the ARL problem.

Typical instantiations of ARL represent these entities through non-linear functions in the form of deep neural networks (DNNs) and formulate parameter learning as a minimax optimization problem. Practically, optimization is performed through simultaneous gradient descent, wherein, small gradient steps are taken concurrently in the parameter space of the encoder, predictor and proxy adversary. The solutions thus obtained have been effective in learning data representations with controlled invariance across applications such as image classification \cite{roy2019mitigating}, multi-lingual machine translation \cite{xie2017controllable} and domain adaptation \cite{ganin2015unsupervised}.

Despite its practical promise, the aforementioned ARL setup suffers from a number of drawbacks:

\vspace{3pt}
\noindent\textbf{--} The minimax formulation of ARL leads to an optimization problem that is non-convex in the parameter space, both due to the adversarial loss function as well as due to the non-linear nature of modern DNNs. As we show in this paper, even for simple instances of ARL where each entity is characterized by a linear function, the problem remains non-convex in the parameter space. Similar observations \cite{nagarajan2017gradient} have been made in a different but related context of adversarial learning in generative adversarial networks (GANs) \cite{goodfellow2014generative}.

\vspace{3pt}
\noindent\textbf{--} Current paradigm of simultaneous gradient descent to solve the ARL problem provides no provable guarantees while suffering from instability and poor convergence \cite{roy2019mitigating, madras2018learning}. Again, similar observations on such limitations have been made \cite{mescheder2017numerics,nagarajan2017gradient} in the context of GANs.

\vspace{3pt}
\noindent\textbf{--} In applications of ARL related to fairness, accountability and transparency of machine learning models, it is critically important to provide performance bounds in addition to empirical evidence of model efficacy. A major shortcoming of existing DNN based ARL solutions is the lack of theoretical analysis or provable bounds on achievable utility and fairness.

In this paper, we take a step back and analytically study the simplest version of the ARL problem from an optimization perspective with the goal of addressing the aforementioned drawbacks. Doing so enables us to delineate the contributions of the expressivity of the entities in ARL (i.e., shallow vs deep models) and the challenges of optimizing the parameters (i.e., local optima through simultaneous gradient descent vs global optima).

\begin{figure}
    \centering
    \includegraphics[width=0.5\textwidth]{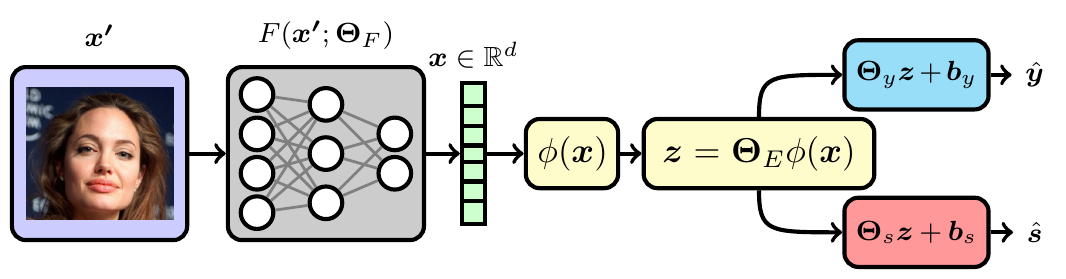}
    \caption{\textbf{Overview:} Illustration of adversarial representation learning for imparting invariance to a fixed biased pre-trained image representation $\bm{x}=F(\bm{x'};\bm{\Theta}_F)$. An encoder $E$, in the form of a kernel mapping, produces a new representation $\bm{z}$. A target predictor and an adversary, in the form of linear regressors, operate on this new representation. We theoretically analyze this ARL setup to obtain a closed form solution for the globally optimal parameters of the encoder $\bm{\Theta}_E$. Provable bounds on the achievable trade-off between the utility and fairness of the representation are also derived.\label{fig:overview}}
\end{figure}

\vspace{5pt}
\noindent \textbf{Contributions:} We first consider the ``linear" form of ARL, where the encoder is a linear transformation, the target predictor is a linear regressor and proxy adversary is a linear regressor. We show that this Linear-ARL leads to an optimization problem that is both non-convex and non-differentiable. Despite this fact, by reducing it into a set of trace problems on a Stiefel manifold, we obtain an exact closed form solution for the global optima. As part of our solution, we also determine optimal dimensionality of the embedding space. We then obtain analytical bounds (lower and upper) on the target and adversary objectives and prescribe a procedure to explicitly control the maximal leakage of sensitive information. Finally, we extend the Linear-ARL formulation to allow non-linear functions through a kernel extension while still enjoying an exact closed-form solution for the global optima. Numerical experiments on multiple datasets, both small and large scale, indicate that the global optima solution for the linear and kernel formulations of ARL are competitive and sometimes even outperform DNN based ARL trained through simultaneous stochastic gradient descent. Practically, we also demonstrate the utility of Linear-ARL and Kernel-ARL for ``imparting" provable invariance to any biased pre-trained data representation. Figure \ref{fig:overview} provides an overview of our contributions. We refer to our proposed algorithm for obtaining the global optima as Spectral-ARL and abbreviate it as SARL.

\vspace{5pt}
\noindent \textbf{Notation:} Scalars are denoted by regular lower case or Greek letters, \eg $n$, $\lambda$. Vectors are denoted by boldface lowercase letters, \eg $\x$, $\y$. Matrices are uppercase boldface letters, \eg $\bX$. A $k\times k$ identity matrix is denoted by $\bI_k$ or $\bI$.
Centered (mean subtracted w.r.t. columns) data matrix is indicated by ``\textasciitilde", \eg $\tilde{\bX}$. Assume that $\bX$ contains $n$ columns, then  $\tilde{\bX} = \bX \bD $, where $\bD = \bI_n - \frac{1}{n}\one \one^T$ and $\one$ denotes the vector of ones with length of $n$. Given matrix $\bM\in \R^{m\times m}$, we use $\tr[\bM]$ to denote its trace (i.e., the sum of its diagonal elements); its Frobenius norm is denoted by $\|\bM\|_F$, which is related to the trace as $\|\bM\|_F^2 =\tr[\bM \bM^T]=\tr[\bM^T \bM]$. The subspace spanned by the columns of $\bM$ is denoted by $\mathcal {R}(\bM)$ or simply  $\M$ (in calligraphy); the orthogonal complement of $\M$ is denoted by $\M^\perp$. The null space of $\bM$ is denoted by $\N(\bM)$. The orthogonal projection onto $\M$ is $P_\M=\bM(\bM^T \bM)^\dagger \bM^T$, where superscript ``$^\dagger$" indicates the Moore-Penrose pseudo inverse ~\cite{laub2005matrix}.

Let $\x \in \R^{d}$ be a random vector. We denote its expectation by $\E[\x]$, and its covariance matrix by $\bC_{x} \in \R^{d\times d}$ as $\bC_{x}=\E\big[(\x-\E[\x])(\x-\E[\x])^T\big]$. Similarly, the cross-covariance $\bC_{xy}\in\R^{d\times r}$ between $\x\in\R^d$ and $\y\in\R^r$ is denoted as $\bC_{xy}=\E\big[(\x-\E[\x])(\y-\E[\y])^T\big]$.

For a $d\times d$ positive definite matrix $\bC\succ0$, its Cholesky factorization results in a full rank matrix $\bQ\in \R^{d\times d}$ such that
\begin{equation}\label{Cholesky}
\bC = \bQ^T \bQ
\end{equation}

%% file: related-work.tex
% related-work.tex

\section{Prior Work}

\vspace{5pt}
\noindent \textbf{Adversarial Representation Learning:} In the context of image classification, adversarial learning has been utilized to obtain representations that are invariant across domains \cite{ganin2015unsupervised, ganin2016domain, tzeng2017adversarial}. Such representations allow classifiers that are trained on a source domain to generalize to a different target domain. In the context of learning fair and unbiased representations, a number of approaches \cite{edwards2015censoring,zhang2018mitigating,beutel2017data,xie2017controllable,mirjalili2018semi,roy2019mitigating,bertran2019adversarially} have used and argued~\cite{madras2018learning} for explicit adversarial networks\footnote{Proxies at training to mimic unknown adversaries during inference.}, to extract sensitive attributes from the encoded data. With the exception of \cite{roy2019mitigating} all the other methods are set up as a minimax game between the encoder, a target task and the adversary. The encoder is setup to achieve fairness by maximizing the loss of the adversary i.e. minimizing negative log-likelihood of sensitive variables as measured by the adversary. Roy \etal~\cite{roy2019mitigating} identify and address the instability in the optimization in the zero-sum minimax formulation of ARL and propose an alternate non-zero sum solution, demonstrating significantly improved empirical performance. All the above approaches use deep neural networks to represent the ARL entities, optimize their parameters through simultaneous stochastic gradient descent, and rely on empirical validation. However, none of them seek to study the nature of the ARL formulation itself i.e., in terms of decoupling the role of the expressiveness of the models and convergence/stability properties of the optimization tools for learning the parameters of said models. Therefore, we seek to bridge this gap by studying simpler forms of ARL from a global optimization perspective.

\vspace{5pt}
\noindent\textbf{Privacy, Fairness and Invariance:} Concurrent work on learning fair or invariant representations of data included an encoder and a target predictor but did not involve an explicit adversary. The role of the adversary is played by an explicit hand designed objective that, typically, competes with that of the target task. The concept of learning fair representations was first introduced by Zemel \etal~\cite{zemel2013learning}. The goal was to learn a representation of data by ``fair clustering" while maintaining the discriminative features of the prediction task. Building upon this work, many techniques have been proposed to learn an unbiased representation of data while retaining its effectiveness for a prediction task. These include the Variational Fair Autoencoder~\cite{louizos2015variational} and the more recent information bottleneck based objective by Moyer \etal~\cite{moyer2018invariant}. As with the ARL methods above, these approaches rely on empirical validation. Neither of them study their respective non-convex objectives from an optimization perspective, nor do they provide any provable guarantees on achievable trade-off between fairness and utility. The competing nature of the objectives considered in this body of work shares resemblance to the non-convex objectives that we study in this paper. Though it is not our focus, the approach presented here could potentially be extended to analyze the aforementioned methods.

\vspace{5pt}
\noindent\textbf{Optimization Theory for Adversarial Learning:} The non-convex nature of the ARL formulation poses unique challenges from an optimization perspective. Practically, the parameters of the models in ARL are optimized through stochastic gradient descent, either jointly \cite{edwards2015censoring, mescheder2017numerics} or alternatively \cite{ganin2015unsupervised}, with the former being a generalization of gradient descent. While the convergence properties of gradient descent and its variants are well understood, there is relatively little work on the convergence and stability of simultaneous gradient descent in adversarial minimax problems. Recently, Mescheder \etal~\cite{mescheder2017numerics} and Nagarajan \etal~\cite{nagarajan2017gradient} both leveraged tools from non-linear systems theory \cite{khalil1996nonlinear} to analyze the convergence properties of simultaneous gradient descent, in the context of GANs, around a given equilibrium. They show that without the introduction of additional regularization terms to the objective of the zero-sum game, simultaneous gradient descent does not converge. However, their analysis is restricted to the two player GAN setting and is not concerned with its global optima.

In the context of fair representation learning, Komiyama \etal~\cite{komiyama2018nonconvex} consider the problem of enforcing fairness constraints in linear regression and provide a solution to obtain the global optima of the resulting non-convex problem. While we derive inspiration from this work, our problem setting and technical solution are both notably different. Specifically, their approach does not involve, (1) an explicit adversary as a measure of sensitive information in the representation, and (2) an encoder tasked with disentangling and discarding the sensitive information in the data.

%% file: approach.tex
% approach.tex

\section{Adversarial Representation Learning}

Let the data matrix $\bX=[\x_1,\dots, \x_n] \in {\R}^{d\times n}$ be $n$ realizations of $d$-dimensional data $\x\in \R^d$. Assume that $\x$ is associated with a sensitive attribute $\s\in\R^q$  and a target attribute $\y\in \R^p$. We denote $n$ realizations of sensitive and target attributes as $\bS=[\s_1,\cdots,\s_n]$ and $\bY=[\y_1,\cdots,\y_n]$, respectively. Treating the attributes as vectors enables us to consider both multi-class classification and regression under the same setup.

\subsection{Problem Setting}
The adversarial representation learning problem is formulated with the goal of learning parameters of an embedding function $E(\cdot;\bt_E): \x \mapsto \z$ with two objectives: (i) aiding a target predictor $T(\cdot;\bt_y)$ to accurately infer the target attribute $\y$ from $\z$, and (ii) preventing an adversary $A(\cdot;\bt_s)$ from inferring the sensitive attribute $\s$ from $\z$. The ARL problem can be formulated as,
\begin{equation}
    \begin{aligned}
        \min_{\bt_E} \min_{\bt_y} & \ \ \mathcal{L}_y\left( T(E(\x;\bt_E);\bt_y), \y \right) \\
        \mathrm {s.t. \ \  } & \min_{\bt_s} \ \mathcal{L}_s\left( A(E(\x;\bt_E);\bt_s), \s \right) \geq \alpha
    \end{aligned}
    \label{eq:arl}
\end{equation}
\noindent where $\mathcal{L}_y$ and $\mathcal{L}_s$ are the loss functions (averaged over training dataset) for the target predictor and the adversary, respectively, $\alpha \in [0,\infty)$ is a user defined value that determines the minimum tolerable loss for the adversary on the sensitive attribute, and the minimization in the constraint is equivalent to the encoder operating against an optimal adversary. Existing instances of this problem adopt deep neural networks to represent $E$, $T$ and $A$ and learn their respective parameters $\{\bt_E, \bt_y, \bt_s\}$ through simultaneous SGD.

\subsection{The Linear Case}\label{sec-linear}
We first consider the simplest form of the ARL problem and analyze it from an optimization perspective. We model both the adversary and the target predictors as linear regressors,
\begin{equation}\label{estimators}
\yh = \bt_y \z + \bb_y, \qquad  \sh = \bt_s \z + \bb_s
\end{equation}
\noindent where $\z$ is an encoded version of $\x$, and $\yh$ and $\sh$ are the predictions corresponding to the target and sensitive attributes. We also model the encoder through a linear mapping,
\begin{equation}\label{encoder}
\bt_E\in \R^{r\times d} \quad  \colon \quad  \x\mapsto \z=\bt_E \x
\end{equation}
\noindent where $r < d$ is the dimensionality\footnote{When $r$ is equal to $d$, the encoder will be unable to guard against the adversary who can simply learn to invert $\bt_E$.} of the projected space. While existing DNN based solutions select $r$ on an ad-hoc basis, our approach for this problem determines $r$ as part of our solution to the ARL problem. For both adversary and target predictors, we adopt the mean squared error (MSE) to assess the quality of their respective predictions i.e., $\mathcal{L}_y(\y,\yh)=\E[\|\y-\yh\|^2]$ and $\mathcal{L}_s(\s,\sh)=\E[\|\s-\sh\|^2]$.

\subsubsection{Optimization Problem}
For any given encoder $\bt_E$ the following Lemma\footnote{We defer the proofs of all lemmas and theorems to the appendix.} gives the minimum MSE for a linear regressor in terms of covariance matrices
and $\bt_E$. The following Lemma assumes that $\x$ is zero-mean and the covariance matrix $\bC_x$ is positive definite. These assumptions are not restrictive since we can always remove the mean and dependent features from $\x$.

\begin{lemma}\label{th1}
Let $\x$ and $\bm{t}$ be two random vectors with $\E[\x]=0$,  $\E[\bm{t}]=\bb$, and $\bC_x\succ0$. Consider a linear regressor, $\bm{\hat{t}} = \bW \z + \bb$, where $\bW \in \R^{m \times r}$ is the parameter matrix, and $\z\in\R^r$ is an encoded version of $\x$ for a given $\bt_E$: $\x  \mapsto \z =\bt_E \x, \quad \bt_E \in \R^{r\times d}$. The minimum MSE that can be achieved by designing $\bW$ is,
\[\min_{\bW} \E[\|\bm{t} - \bm{\hat{t}}\|^2 ]  = \tr \big[\bC_t \big] - \big\|P_\M \bQ_x^{-T} \bC_{xt} \big\|_F^2
\]
where $\bM = \bQ_x \bt_E^T \in \R^{d\times r}$, and $ \bQ_x \in \R^{d\times d}$ is a Cholesky factor of $\bC_x$ as shown in (\ref{Cholesky}).
\end{lemma}

Applying this result to the target and adversary regressors, we obtain their minimum MSEs,
\begin{equation}
  \begin{aligned}
    \J_y (\bt_E) &=  \min_{\bt_y} \mathcal{L}_y\left( T(E(\x;\bt_E);\bt_y), \y \right) \\
    &= \tr\big [\bC_y \big] - \big\|P_\M \bQ_x^{-T} \bC_{xy}\big\|_F^2
  \end{aligned}
  \label{eq:cod-y}
\end{equation}
\begin{equation}
  \begin{aligned}
    \J_s (\bt_E) &= \min_{\bt_s} \mathcal{L}_s\left( A(E(\x;\bt_E);\bt_s), \s \right) \\
    &= \tr\big [\bC_s \big] - \big\|P_\M \bQ_x^{-T} \bC_{xs}\big\|_F^2
\end{aligned}
\label{eq:cod-s}
\end{equation}

Given the encoder, $\J_y(\bt_E)$ is related to the performance of the target predictor; whereas $\J_s(\bt_E)$ corresponds to the amount of sensitive information that an adversary is able to leak. Note that the linear model for $T$ and $A$ enables us to obtain their respective optimal solutions for a given encoder $\bt_E$. On the other hand, when $T$ and $A$ are modeled as DNNs, doing the same is analytically infeasible and potentially impractical.

The orthogonal projector $P_\M$ in Lemma~\ref{th1} is a function of two factors, a data dependent term $\bQ_x$ and the encoder parameters $\bt_E$. While the former is fixed for a given dataset, the latter is our object of interest. Pursuantly, we decompose $P_\M$ in order to separably characterize the effect of these two factors.  Let the columns of $\bL_x\in\R^{d\times d}$ be an orthonormal basis for the column space of $\bQ_x$. Due to the bijection $\bG_E= \bL_x^{-1} \bQ_x \bt_E^T  \Leftrightarrow \bt_E =\bG_E^T \bL_x^T \bQ_x^{-T}$ from $\bL_x\bG_E=\bQ_x\bt_E^T$, determining the encoder parameters $\bt_E$ is equivalent to determining $\bG_E$. The projector $P_\M$ can now be expressed in terms of $P_\G$, which is only dependent on the free parameter $\bG_E$.
\begin{equation}\label{A-G}
P_\M  =  \bM \big(\bM^T \bM \big)^\dagger \bM^T  = \bL_x P_\G \bL_x^T
\end{equation}
where we used the equality $\bM=\bQ_x \bt_E^T$ and the fact that $\bL_x ^T \bL_x=\bI$.

Now, we turn back to the ARL setup and see how the above decomposition can be leveraged. The optimization problem in (\ref{eq:arl}) reduces to,
\begin{equation}
\begin{aligned}
\min_{\bG_E} & \ \ \J_y (\bG_E)  \\
\mathrm {s.t. \ \  } & \J_s (\bG_E) \ge \alpha
\end{aligned}
\label{eq-constrained}
\end{equation}
\noindent where the minimum MSE measures of~(\ref{eq:cod-y}) and~(\ref{eq:cod-s}) are now expressed in terms of $\bG_E$ instead of $\bt_E$.

Before solving this optimization problem, we will first interpret it geometrically. Consider a simple example where $\x$ is a white random vector i.e., $\bC_x=\bI$. Under this setting, $\bQ_x = \bL_x = \bI$ and $\bG_E = \bt_E^T$. As a result, the optimization problem in (\ref{eq-constrained}) can alternatively be solved in terms of $\bG_E = \bt^T_E$ as $\J_y(\bG_E) = \tr\big [\bC_y \big] -   \big \|P_{ \G} \bC_{xy}\big\|_F^2\nn$ and $\J_s(\bG_E) = \tr\big [\bC_s \big] -  \big \|P_{ \G} \bC_{xs}\big\|_F^2\nn$.

The constraint $\J_s(\bG_E) \geq \alpha$ implies $\big \|P_{ \G} \bC_{xs}\big\|_F^2\le \big(\tr\big [\bC_s \big]-\alpha\big)$ which is geometrically equivalent to the subspace $\G$ being outside (or tangent to) the cone around $\bC_{xs}$. Similarly, minimizing $\J_y(\bG_E)$ implies maximizing $\big \|P_{ \G} \bC_{xy}\big\|_F^2\nn$, which in turn is equivalent to minimizing the angle between the subspace $\G$ and the vector $\bC_{xy}$. Therefore, the global optima of (\ref{eq-constrained}) is any hyper plane $\G$ which is outside the cone around $\bC_{xs}$ while subtending the smallest angle to $\bC_{xy}$. An illustration of this setting and its solution is shown in Figure~\ref{fig-intuition} for $d=3$, $r=2$ and $p=q=1$.
\begin{figure}[t]
\centering
\includegraphics[scale=0.7]{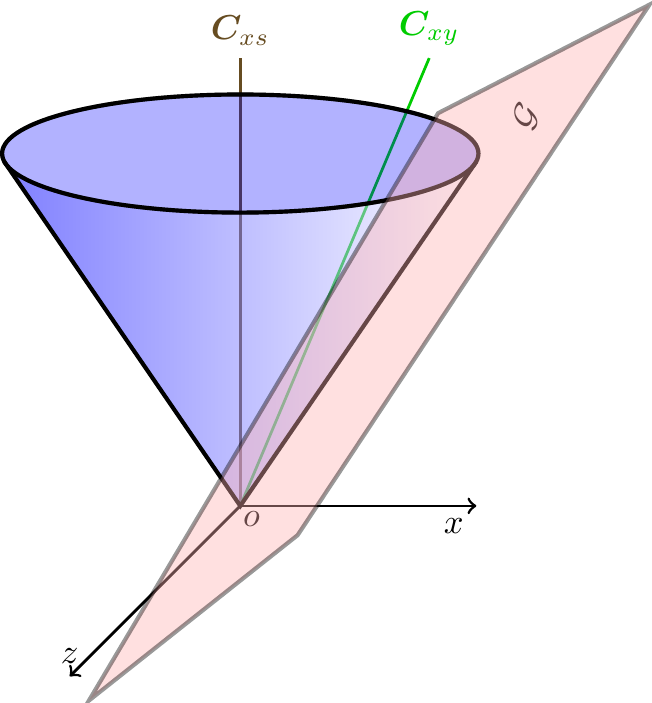}
\caption{\textbf{Geometric Interpretation:} An illustration of a three-dimensional input space $\x$ and one-dimensional target and adversary regressors. Therefore, both $\bC_{xs}$ and $\bC_{xy}$ are one-dimensional. We locate the $y$-axis in the same direction as $\bC_{xs}$. The feasible space for the solution $\bG_E=\bt_E^T$ imposed by the constraint $\J_s(\bt_E) \ge \alpha$ corresponds to the region \emph{outside} the cone (specified by $\bC_s$ and $\alpha$) around $\bC_{xs}$. The non-convexity of the problem stems from the non-convexity of this feasible set. The objective $\min \J_y(\bt_E)$ corresponds to minimizing the angle between the line $\bC_{xy}$ and the plane $\G$. When $\bC_{xy}$ is outside the cone, the line $\bC_{xy}$ itself or any plane that contains the line $\bC_{xy}$ and does not intersect with the cone, is a valid solution. When $\bC_{xy}$ is inside the cone, the solution is either a line or, as we illustrate, a tangent hyperplane to the cone that is closest to $\bC_{xy}$. The non-differentiability stems from the fact that the solution can either be a plane or a line.}
\label{fig-intuition}
\end{figure}

Constrained optimization problems such as (\ref{eq-constrained}) are commonly solved through their respective unconstrained Lagrangian~\cite{bertsekas1999nonlinear} formulations as shown below
\begin{equation}\label{eq-main}
\min_{\bG_E \in \R^{d\times r}} \Big\{ (1-\lambda) \J_y (\bG_E)-(\lambda) J_s (\bG_E) \Big\}
\end{equation}
for some parameter $0 \le \lambda \le 1$. Such an approach affords two main advantages and one disadvantage; (a) A direct and closed-form solution can be obtained. (b) Framing (\ref{eq-main}) in terms of $\lambda$ and $(1-\lambda)$ allows explicit control between the two extremes of \emph{no privacy} ($\lambda=0$) and \emph{no target} ($\lambda=1$). As a consequence, it can be shown that for every $\lambda \in [0,1], \mbox{ } \exists \mbox{ } \alpha \in [\alpha_\min, \alpha_\max]$ (see Section \ref{sec:relation} of appendix for a proof). In practice, given a user specified value of $\alpha_\min\le\alpha_{\mathrm{tol}} \le \alpha_\max$ , we can solve~(\ref{eq-constrained}) by iterating over $\lambda\in[0,1]$ until the solution of~(\ref{eq-main}) yields the same specified $\alpha_{\mathrm{tol}}$. (c) The vice-versa on the other hand does not necessarily hold i.e., for a given tolerable loss $\alpha$ there may not be a corresponding $\lambda \in [0,1]$. This is the theoretical limitation\footnote{Practically, as we show in Figures \ref{fig:bounds-gauss} and \ref{fig:bounds-cifar}, all values of $\alpha\in[\alpha_{min},\alpha_{max}]$ appear to be reachable as we sweep through $\lambda\in[0,1]$.} of solving Lagrangian problem instead of the constrained problem.

Before we obtain the solution to the Lagrangian formulation (\ref{eq-main}), we characterize the nature of the optimization problem in the following theorem.
\begin{theorem}\label{th2}
As a function of $\bG_E \in \R^{d\times r}$, the objective function in~(\ref{eq-main})  is neither convex nor differentiable.
\end{theorem}

\subsubsection{Learning}
\label{sec-solution}
Despite the difficulty associated with the objective in~(\ref{eq-main}), we derive a closed-form solution for its global optima. Our key insight lies in partitioning the search space $\R^{d\times r}$ based on the rank of the matrix $\bG_E$. For a given rank $i$, let $\mathcal{S}_i$ be the set containing all matrices $\bG_E$ of rank $i$,
\begin{equation}
\mathcal{S}_i = \big\{ \bG_E\in \R^{d\times r}  \quad \big|  \quad \mathrm{rank}(\bG_E)=i\big\}, \quad i=0,1,\cdots,r \nn
\end{equation}
Obviously, $\bigcup_{i = 0}^{r} \mathcal{S}_i = \R^{d\times r}$. As a result, the optimization problem in~(\ref{eq-main}) can be solved by considering $r$ minimization problems, one for each possible rank of $\bG_E$:
\begin{equation}\label{eq-main2}
\min_{i\in \{1,\dots,r\}}\Big\{\min_{\bG_E \in \mathcal{S}_i} (1-\lambda) \J_y (\bG_E)-(\lambda) J_s (\bG_E)\Big\}
\end{equation}

We observe from~(\ref{eq:cod-y}),~(\ref{eq:cod-s}) and~(\ref{A-G}), that the optimization problem in~(\ref{eq-main}) is dependent only on a subspace $\G$. As such, the solution $\bG_E$ is not unique since many different matrices can span the same subspace. Hence, it is sufficient to solve for any particular $\bG_E$ that spans the optimal subspace $\G$. Without loss of generality we seek an orthonormal basis spanning the optimal subspace $\G$ as our desired solution. We constrain $\bG_E\in\R^{d\times i}$ to be an orthonormal matrix i.e., $\bG_E^T\bG_E=\bI_i$ where $i$ is the dimensionality of $\G$. Ignoring the constant terms in $J_y$ and $J_s$, for each $i =1, \ldots,r$, the minimization problem over $\mathcal{S}_i$ in~(\ref{eq-main2}) reduces to,
\begin{equation}\label{eq-main3}
\min_{\bG_E^T \bG_E = \bI_i} J_\lambda(\bG_E)
\end{equation}
 where
 \begin{equation}
   \begin{aligned}
     J_\lambda(\bG_E) &= \lambda\|\bL_x \bG_E \bG_E^T\bL_x^T\bQ_x^{-T}\bC_{xs}\|_F^2\nn\\
     &-(1-\lambda)\|\bL_x \bG_E \bG_E^T \bL_x^T\bQ_x^{-T}\bC_{xy}\|_F^2
  \end{aligned}
\end{equation}

From basic properties of trace, we have, $J_\lambda(\bG_E) = \tr\big[\bG_E^T \bB \bG_E\big]$ where $\bB\in\R^{d\times d}$ is a symmetric matrix:
\begin{equation}\label{eq-B}
\bB = \bL_x^T  \bQ_x^{-T}\big({\lambda}\,\bC^T_{sx} \bC_{sx} - {(1-\lambda)}\,\bC^T_{yx} \bC_{yx}\big) \bQ_x^{-1} \bL_x
\end{equation}
The optimization problem in~(\ref{eq-main3}) is equivalent to trace minimization on a Stiefel manifold which has closed-form solution(s) (see \cite{kokiopoulou2011trace} and \cite{edelman1998geometry}).

In view of the above discussion the solution to the optimization problem in~(\ref{eq-main}) or equivalently~(\ref{eq-main2}) can be stated in the next theorem.
\begin{theorem}\label{th3}
Assume that the number of negative eigenvalues ($\beta$) of $\bB$ in~(\ref{eq-B}) is $j$. Denote $\gamma=\min\{r, j \}$. Then, the minimum value in~(\ref{eq-main2}) is given as,
\begin{equation}\label{eq-beta}
\beta_1 + \beta_{2}+\cdots +\beta_{\gamma}
\end{equation}
where $\beta_1 \leq \beta_2 \leq \ldots \leq \beta_{\gamma} < 0$ are the $\gamma$ smallest eigenvalues of $\bB$. And the minimum can be attained by $\bG_E =\bV$, where the columns of $\bV$ are  eigenvectors corresponding to all the $\gamma$ negative eigenvalues of $\bB$.
\end{theorem}
Note that, including the eigenvectors corresponding to zero eigenvalues of $\bB$ into our solution $\bG_E$ in Theorem~\ref{th3} does not change the minimum value in~(\ref{eq-beta}). But, considering only negative eigenvectors results in $\bG_E$ with the least rank and thereby an encoder that is less likely to contain sensitive information for an adversary to exploit. Once $\bG_E$ is constructed, we can obtain our desired encoder as, $\bt_E = \bG_E^T \bL_x^T \bQ_x^{-T}$. Recall that the solution in Theorem~\ref{th3} is under the assumption that the covariance $\bC_x$ is a full-rank matrix. In Section \ref{sec:linear} of the appendix, we develop a solution for the more practical and general case where empirical moments are used instead.

\subsection{Non-Linear Extension Through Kernelization\label{kernelization}}
We extend the ``linear" version of the ARL problem studied thus far to a ``non-linear" version through kernelization. We model the encoder in the ARL problem as a linear function over non-linear mapping of inputs as illustrated in Figure~\ref{fig:overview}. Let the data matrix $\bX$ be mapped non-linearly by a possibly unknown and infinite dimensional function $\phi_x (\cdot)$ to $\bp_x$. Let the corresponding reproducing kernel function be $k_x(\cdot,\cdot)$. The centered kernel matrix can be obtained as,
\begin{equation}
\tilde{\bK}_x = \tilde{\bp}_x^T \tilde{\bp}_x = \bD^T \bp_x^T \bp_x \bD = \bD^T \bK_x \bD
\end{equation}
\noindent where $\bK_x$ is the kernel matrix on the original data $\bX$.

If the co-domain of $\phi_x(\cdot)$ is infinite dimensional (e.g., RBF kernel), then the encoder in~(\ref{encoder}) would be also be infinite dimensional i.e., $\bt_E\in\R^{r\times \infty}$, which is infeasible to learn directly. However, the representer theorem~\cite{shawe2004kernel} allows us to construct the encoder as a linear function of $\tilde{\bp}^T $, i.e, $\bt_E = \bl \tilde{\bp}_x^T = \bl \bD^T \bp_x^T$. Hence, a data sample $\bm{x}$ can be mapped through the ``kernel trick" as,
\begin{equation}
  \begin{aligned}
    \bt_E\phi_x(\x) &= \bl \bD^T\bp_x^T \phi_x(x) \\
    &= \bl\bD^T [k_x(\x_1,\x),\,\cdots,\,k_x(\x_n,\x)]^T
\end{aligned}
\end{equation}

Hence, designing $\bt_E$ is equivalent to designing $\bl\in\R^{r\times n}$. The Lagrangian formulation of this Kernel-ARL setup and its solution shares the same form as that of the linear case (\ref{eq-main}). The objective function remains non-convex and non-differentiable, while the matrix $\bB$ is now dependent on the kernel matrix $\bK_x$ as opposed to the covariance matrix $\bC_x$ (see Section \ref{sec:kernel} of the appendix for details).
\begin{equation}
\bB =  \bL_x^T \big(\lambda\,\tilde{\bS}^T\tilde{\bS}-(1-\lambda)\,\tilde{\bY}^T\tilde{\bY}\big) \bL_x
\end{equation}
where the columns of $\bL_x$ are the orthonormal basis for $\tilde{\bK}_x$. Once $\bG_E$ is obtained through the eigendecomposition of $\bB$, we can find $\bl$ as $\bl = \bG_E^T \bL_x^T \tilde{\bK}_x^\dagger$. This non-linear extension in the form of kernelization serves to study the ARL problem under a setting where the encoder possess greater representational capacity while still being able to obtain the global optima and bounds on objectives of the target predictor and the adversary as we show next. Algorithm~\ref{alg1} provides a detailed procedure for solving both the Linear-ARL and Kernel-ARL formulations.

\section{Analytical Bounds}
In this section we introduce bounds on the utility and invariance of the representation learned by SARL. We define four bounds $\alpha_\min$, $\alpha_\max$, $\gamma_\min$ and $\gamma_\max$.

\vspace{5pt}
\noindent$\bm{\gamma_\min:}$ A lower bound on the minimum achievable target loss, or equivalently an upper bound on the best achievable target performance. This bound can be expressed as the minimum target MSE across all possible encoders $\bt_E$ and is attained at $\lambda=0$.
\begin{equation}
    \gamma_\min = \min_{\bt_E}J_y(\bt_E)\nn
\end{equation}

\vspace{2pt}
\noindent$\bm{\alpha_\max:}$ A upper bound on the maximum achievable adversary loss, or equivalently a lower bound on the minimum leakage of sensitive attribute. This bound can be expressed as the maximum adversary MSE across all possible encoders $\bt_E$ and is attained at $\lambda=1$.
\begin{equation}
    \alpha_\max = \max_{\bt_E} J_s(\bt_E)\nn
\end{equation}

\vspace{2pt}
\noindent$\bm{\gamma_\max:}$ An upper bound on the maximum achievable target loss, or equivalently a lower bound on the minimum achievable target performance. This bound corresponds to the scenario where the encoder is constrained to maximally hinder the adversary. In all other cases one can obtain higher target performance by choosing a better encoder. This bound is attained in the limit $\lambda \rightarrow 1$ and can be expressed as,
\begin{equation}
    \gamma_\max = \min_{\arg \max J_s(\bt_E)} J_y(\bt_E)\nn
\end{equation}

\vspace{2pt}
\noindent$\bm{\alpha_\min:}$ A lower bound on the minimum achievable adversary loss, or equivalently an upper bound on the maximum leakage of sensitive attribute. The absolute lower bound is obtained in the scenario where the encoder is neither constrained to aid the target nor hinder the adversary i.e.,
\begin{equation}
    \alpha^*_{min} = \min_{\bt_E} J_s(\bt_E)\nn
\end{equation}
However, this is an unrealistic scenario since in the ARL problem, by definition, the encoder is explicitly designed to aid the target. Therefore, a more realistic lower bound can be defined under the constraint that the encoder maximally aids the target i.e.,
\begin{equation}
    \bar{\alpha}_{\min} = \min_{\arg \min J_y(\bt_E)} J_s(\bt_E)\nn
\end{equation}
However, even this bound is not realistic, since, among all the encoders that aid the target one can always choose the encoder that minimizes the leakage of the sensitive attribute. The bound corresponding to such an encoder can be expressed as,
\begin{equation}
    \alpha_\min = \max_{\arg \min J_y(\bt_E)} J_s(\bt_E)\nn
\end{equation}
This bound is attained in the limit $\lambda \rightarrow 0$. It is easy to see that these bounds are ordinally related as,
\begin{equation}
    \alpha^*_\min \leq \bar{\alpha}_{\min} \leq \alpha_\min \nn
\end{equation}
To summarize, in each of these cases, there exists an encoder that achieves the respective bound. Therefore, given a choice, the encoder that corresponds to $\alpha_\min$ is the most desirable.

The following Lemma defines these bounds and their respective closed form expressions as a function of data.
\begin{lemma}\label{th4}
Let the columns of $\bL_x$ be the orthonormal basis for $\tilde{\bK}_x$ (in linear case $\tilde{\bK}_x = \tilde{\bX}^T \tilde{\bX}$). Further, assume that the columns of $\bV_s$ are the singular vectors corresponding to zero singular values of $\tilde{\bS}\bL_x$
and the  columns of $\bV_y$ are the singular vectors corresponding to non-zero singular values of $\tilde{\bY}\bL_x$.
Then, we have
\begin{align}
    \gamma_\min = &\min_{\bt_E}J_y(\bt_E)\nn\\
    = &\frac{1}{n}\big\|\tilde{\bY}^T\big\|_F^2-\frac{1}{n}{\|\tilde{\bY}\bL_x\|_F^2} \nn\\
    \gamma_\max = &\min_{\arg \max J_s(\bt_E)} J_y(\bt_E)\nn\\
        = &\frac{1}{n}\big\|\tilde{\bY}^T\big\|_F^2 - \frac{1}{n} \big\| \tilde{\bY} \bL_x \bV_s\big\|_F^2\nn\\
    \alpha_\min = &\max_{\arg \min J_y(\bt_E)} J_s(\bt_E)\nn\\
    = &\frac{1}{n}\big\|\tilde{\bS}^T\big\|_F^2 - \frac{1}{n} \big\| \tilde{\bS} \bL_x \bV_y\big\|_F^2\nn\\
    \alpha_\max = &\max_{\bt_E} J_s(\bt_E)\nn\\
    = &\frac{1}{n}\big\|\tilde{\bS}^T\big\|_F^2\nn
\end{align}
\end{lemma}

Under the special case of one dimensional data, i.e., $\x$, $\y$ and $\s$ are scalars, the above bounds can be related to the normalized correlation of the variables involved. Specifically, the normalized bounds $\gamma_\min$ and $\alpha_\min$ can be expressed as,
\begin{align}
    &&\frac{\gamma_\min}{\sigma_y^2} = 1-\rho^2(\x, \y) \nn\\
    &&\frac{\alpha_\min}{\sigma_s^2} = 1-\rho^2(\x, \s)\nn
\end{align}
where $\rho(\cdot, \cdot)$ denotes the correlation coefficient (i.e., normalized correlation) between two random variables and $\sigma_y^2 = \E[\tilde{y}^2]$ ($\sigma_s^2$ is similarly defined). Similarly, the upper bounds $\gamma_\max$ and $\alpha_\max$ can be expressed in terms of the variance of the label space as,
\begin{align}
    \frac{\gamma_\max}{\sigma_y^2} = \frac{\alpha_\max}{\sigma_s^2}=1 \nn
\end{align}
Therefore, in the one-dimensional setting, the achievable bounds are related to the underlying alignment between the subspace spanned by the data $\bX$ and the respective subspaces spanned by the labels $\bS$ and $\bY$.
\section{Computational Complexity\label{sec:computation}}
In the case of Linear-SARL, calculating the covariance matrices $\bC_x$, $\bC_{yx}$ and $\bC_{s x}$ requires $\mathcal O(d^2 n)$, $\mathcal O(p^2 n)$ and $\mathcal O(q^2 d)$ multiplications, respectively. Next, the complexity of Cholesky factorization $\bC_x = \bQ^T_x \bQ_x$ and calculating its inverse $\bQ_x^{-1}$ is $\mathcal O(d^3)$ each. Finally, solving the optimization problem has a complexity of $\mathcal O(d^3)$ to eigendecompose the $d\times d$ matrix $\bB$. In the case of Kernel-SARL, eigendecomposition of $\bB$ requires $\mathcal O(n^3)$ operations. However, for scalability i.e., large $n$ (e.g., CIFAR-100), the Nystr\"{o}m method with data sampling \cite{kumar2012sampling} can be adopted. To summarize, the complexity of the linear and kernel formulations is $\mathcal O(d^3)$ and $\mathcal O(n^3)$, respectively.

%% file: experiments.tex
% experiments.tex

\section{Numerical Experiments}

We evaluate the efficacy of the proposed Spectral-ARL (SARL) algorithm in finding the global optima, and compare it with other ARL baselines that are based on the standard simultaneous SGD optimization (henceforth referred to as SSGD). In all experiments we refer to our solution for ``linear" ARL as Linear-SARL and the solution to the ``kernel" version of the encoder with linear classifiers for the predictor and adversary as Kernel-SARL.

\subsection{Mixture of Four Gaussians}
\begin{figure}[ht]
\centering
\tdplotsetmaincoords{0}{0}
\begin{tikzpicture}[scale=.95, ,tdplot_main_coords]
\begin{axis}
[xmin=0,xmax=4, ymin=0,ymax=4, zmin=-1, zmax=1, xticklabels={,,}, yticklabels={,,}, zticklabels={,,}
    ,xmajorgrids=true,
    grid style=dashed
    ,ymajorgrids=true,
    grid style=dashed
    ,zmajorgrids=true,
    grid style=dashed]
        \addplot3[only marks, fill=blue,mark=*] table {figures/tikz3.txt};
        \addplot3[only marks, fill=red,mark=*] table {figures/tikz4.txt};
        \addplot3[only marks,very thick, blue,mark=x] table {figures/tikz1.txt};
        \addplot3[only marks,very thick, red,mark=x] table {figures/tikz2.txt};
\end{axis}
\end{tikzpicture}
\caption{Samples from a mixture of four Gaussians. Each sample has two attributes, shape and color.}
\label{fig:mixdata}
\end{figure}
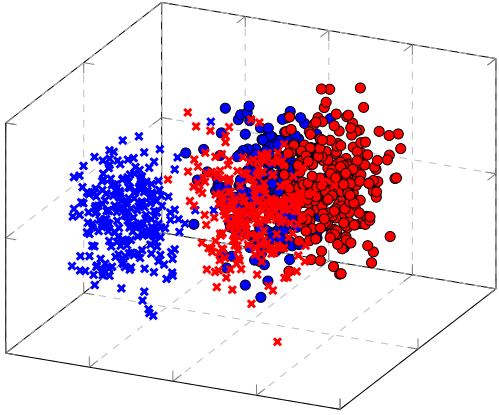
\begin{figure}[ht]
  \centering
  \includegraphics[width=0.45\textwidth]{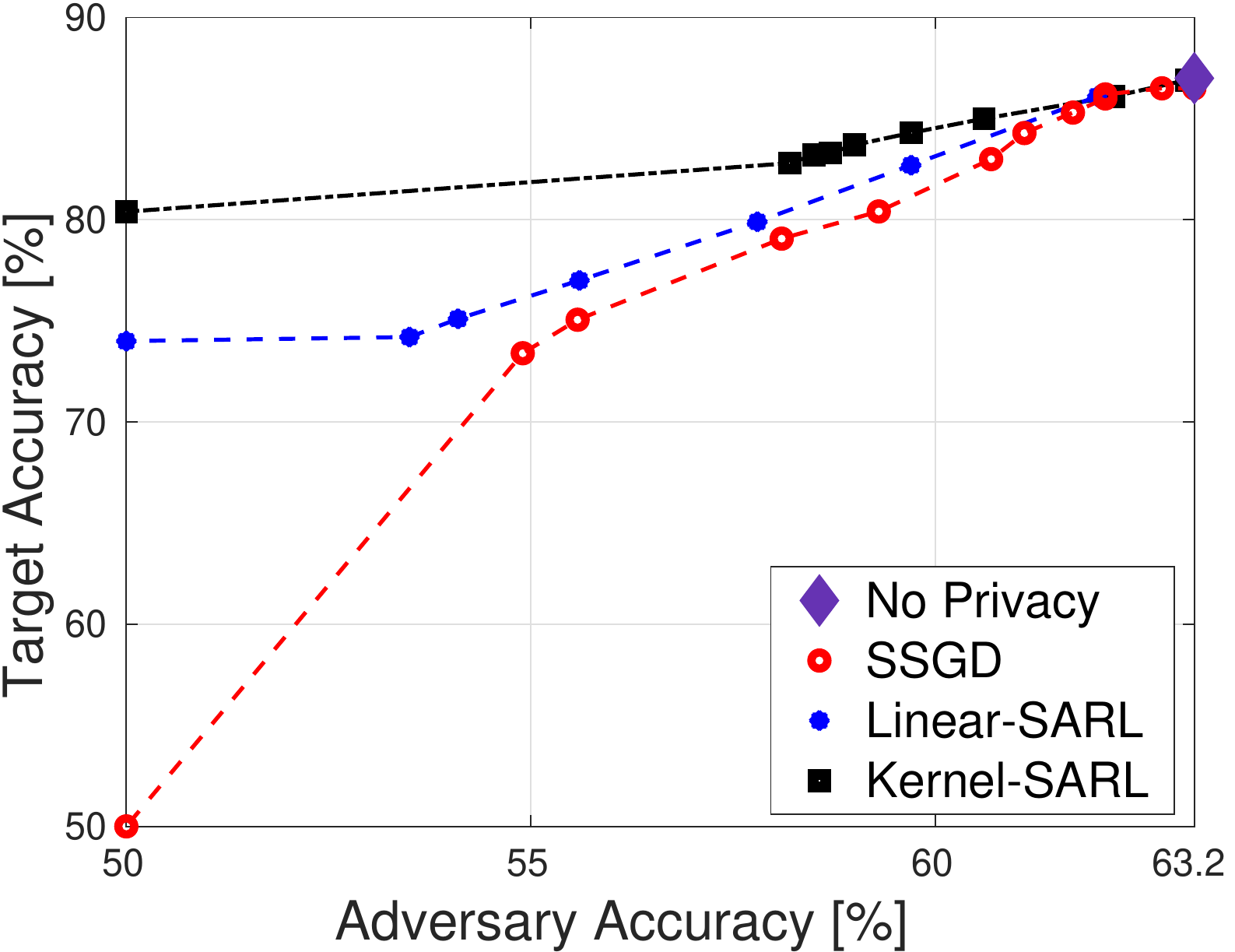}
  \caption{\textbf{Gaussian Mixture:} Trade-off between target performance and leakage of sensitive attribute by adversary.\label{fig-mog}}
\end{figure}
\begin{figure}[ht]
\begin{subfigure}[b]{0.23\textwidth}
    \centering
    \includegraphics[width=\textwidth]{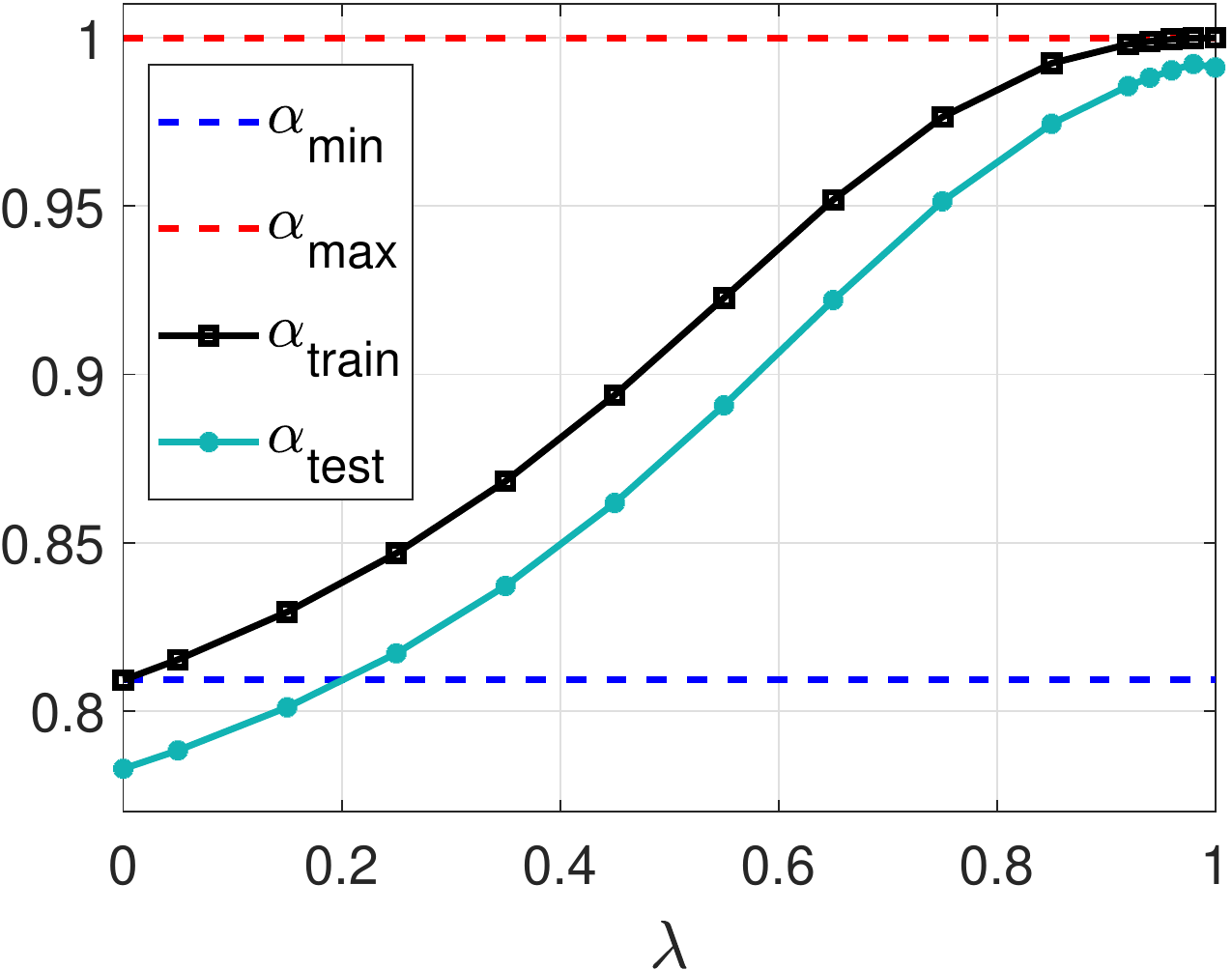}
    \caption{Linear-SARL\label{fig:bounds-gauss-lin}}
\end{subfigure}
\begin{subfigure}[b]{0.23\textwidth}
    \centering
    \includegraphics[width=\textwidth]{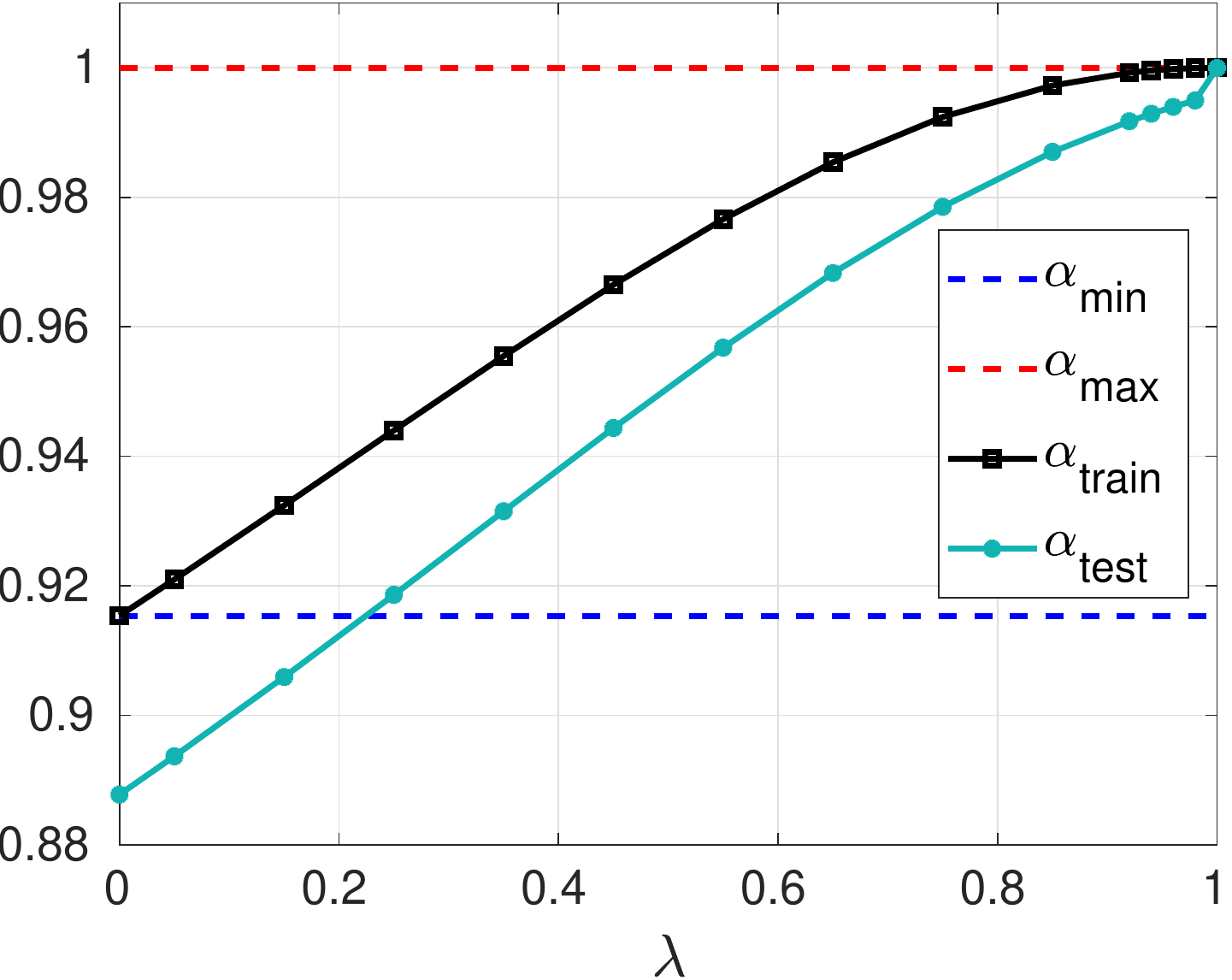}
    \caption{Kernel-SARL\label{fig:bounds-gauss-ker}}
\end{subfigure}
\caption{\textbf{Gaussian Mixture:} Lower and upper bounds on adversary loss, $\alpha_{min}$ and $\alpha_{max}$, computed on training set. The loss achieved by our solution as we vary $\lambda$ is shown on the training and testing sets, $\alpha_{train}$ and $\alpha_{test}$, respectively.\label{fig:bounds-gauss}}
\end{figure}
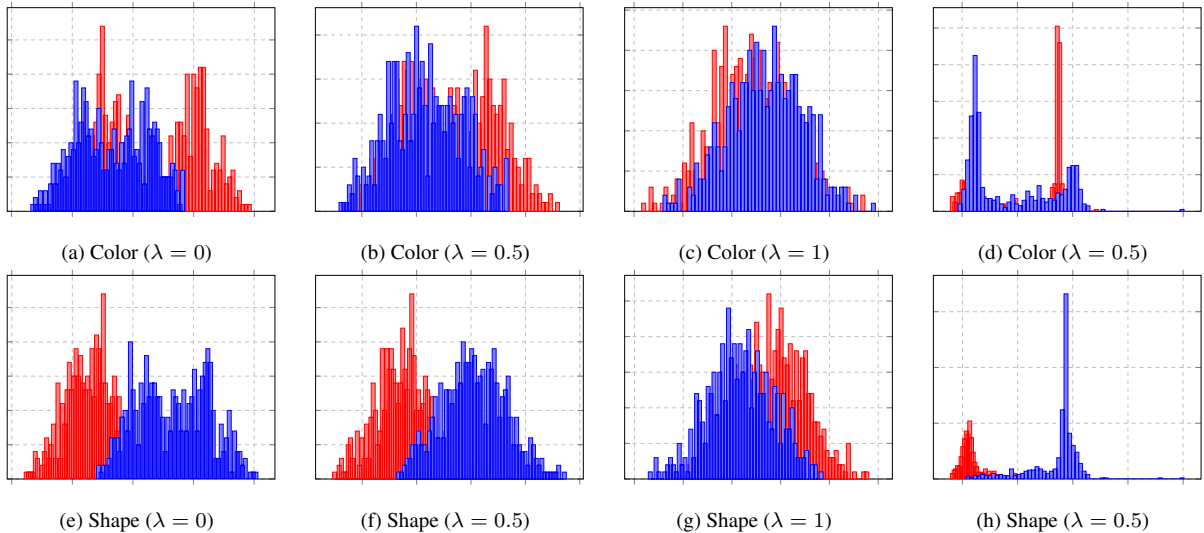
\begin{figure*}[t]
\centering
\begin{subfigure}[b]{0.23\textwidth}
\centering
\begin{tikzpicture}[scale=0.5]
Adversary\begin{axis}[width=8.7cm, height=7cm,
xticklabels={,,}, yticklabels={,,}
    ,xmajorgrids=true,
    grid style=dashed
    ,ymajorgrids=true,
    grid style=dashed,
    ybar,
    ymin=0,
    bar width=3pt
]
\addplot[red, fill=red!50 ] table  {figures/h0_y1.txt};
\addplot[blue, fill=blue!50] table  {figures/h0_s1.txt};
     \end{axis}
\end{tikzpicture}\caption{Color ($\lambda=0$)}
\end{subfigure}
\begin{subfigure}[b]{0.23\textwidth}
\centering
\begin{tikzpicture}[scale=0.5]
Adversary\begin{axis}[width=8.7cm, height=7cm,
xticklabels={,,}, yticklabels={,,}
    ,xmajorgrids=true,
    grid style=dashed
    ,ymajorgrids=true,
    grid style=dashed,
    ,ybar,
    ymin=0,
    bar width=3pt
]
\addplot[red, fill=red!50] table  {figures/h50_y1.txt};
\addplot[blue, fill=blue!50] table  {figures/h50_s1.txt};
     \end{axis}
\end{tikzpicture}\caption{Color ($\lambda=0.5$)}
\end{subfigure}
\begin{subfigure}[b]{0.23\textwidth}
\centering
\begin{tikzpicture}[scale=0.5]
Adversary\begin{axis}[width=8.7cm, height=7cm,
xticklabels={,,}, yticklabels={,,}
    ,xmajorgrids=true,
    grid style=dashed
    ,ymajorgrids=true,
    grid style=dashed,
    ,ybar,
    ymin=0,
    bar width=3pt
]
\addplot[red, fill=red!50] table  {figures/h99_y1.txt};
\addplot[blue, fill=blue!50] table  {figures/h99_s1.txt};
     \end{axis}
\end{tikzpicture}\caption{Color ($\lambda=1$)}
\end{subfigure}
\begin{subfigure}[b]{0.23\textwidth}
\centering
\begin{tikzpicture}[scale=0.5]
\begin{axis}[width=8.7cm, height=7cm,
xticklabels={,,}, yticklabels={,,}
    ,xmajorgrids=true,
    grid style=dashed
    ,ymajorgrids=true,
    grid style=dashed,
    ybar,
    ymin=0,
    bar width=3pt
]
\addplot[red, fill=red!50] table  {figures/h50_kernel_y1.txt};
\addplot[blue, fill=blue!50] table  {figures/h50_kernel_s1.txt};
\end{axis}
\end{tikzpicture}\caption{Color ($\lambda=0.5$)}
\end{subfigure}
\begin{subfigure}[b]{0.23\textwidth}
\centering
\begin{tikzpicture}[scale=0.5]
\begin{axis}[width=8.7cm, height=7cm,
xticklabels={,,}, yticklabels={,,}
    ,xmajorgrids=true,
    grid style=dashed
    ,ymajorgrids=true,
    grid style=dashed,
    ,ybar,
    ymin=0,
    bar width=3pt
]
\addplot[red, fill=red!50] table  {figures/h0_y2.txt};
\addplot[blue, fill=blue!50] table  {figures/h0_s2.txt};
\end{axis}
\end{tikzpicture}\caption{Shape ($\lambda=0$)}
\end{subfigure}
\begin{subfigure}[b]{0.23\textwidth}
\centering
\begin{tikzpicture}[scale=0.5]
\begin{axis}[width=8.7cm, height=7cm,
xticklabels={,,}, yticklabels={,,}
    ,xmajorgrids=true,
    grid style=dashed
    ,ymajorgrids=true,
    grid style=dashed,
    ,ybar,
    ymin=0,
    bar width=3pt
]
\addplot[red, fill=red!50] table  {figures/h50_y2.txt};
\addplot[blue, fill=blue!50] table  {figures/h50_s2.txt};
\end{axis}
\end{tikzpicture}\caption{Shape ($\lambda=0.5$)}
\end{subfigure}
\begin{subfigure}[b]{0.23\textwidth}
\centering
\begin{tikzpicture}[scale=0.5]
\begin{axis}[width=8.7cm, height=7cm,
xticklabels={,,}, yticklabels={,,}
    ,xmajorgrids=true,
    grid style=dashed
    ,ymajorgrids=true,
    grid style=dashed,
    ybar,
    ymin=0,
    bar width=3pt
]
\addplot[red, fill=red!50] table  {figures/h99_y2.txt};
\addplot[blue, fill=blue!50] table  {figures/h99_s2.txt};
\end{axis}
\end{tikzpicture}\caption{Shape ($\lambda=1$)}
\end{subfigure}
\begin{subfigure}[b]{0.23\textwidth}
\centering
\begin{tikzpicture}[scale=0.5]
\begin{axis}[width=8.7cm, height=7cm,
xticklabels={,,}, yticklabels={,,}
    ,xmajorgrids=true,
    grid style=dashed
    ,ymajorgrids=true,
    grid style=dashed,
    ybar,
    ymin=0,
    bar width=3pt
]
\addplot[red, fill=red!50] table  {figures/h50_kernel_y2.txt};
\addplot[blue, fill=blue!50] table  {figures/h50_kernel_s2.txt};
\end{axis}
\end{tikzpicture}\caption{Shape ($\lambda=0.5$)}
\end{subfigure}
\caption{\textbf{Gaussian Mixture:} The optimal dimensionality of embedding $\bm{z}$ is 1. Visualization of the embedding histogram w.r.t each attribute for different relative emphasis, $\lambda$, on the target (shape) and sensitive attributes (color). Top row is color and bottom row is shape. First three columns show results for a linear-encoder. At $\lambda=0$ the weight on the adversary is 0, so color is still separable. As the value of $\lambda$ increases, we observe that the colors are less and less separable. Last column shows results for a kernel-encoder. Visualization of the embedding histogram for $\lambda=0.5$. We observe that the target attribute is quite separable while the sensitive attribute is entangled. \label{fig:gaussian}}
\end{figure*}

We first consider a simple example in order to visualize and compare the learned embeddings from different ARL solutions. We consider a three-dimensional problem where each data sample consists of two attributes, color and shape. Specifically, the input data $\bX$ is generated from a mixture of four different Gaussian distributions corresponding to different possible combinations of the attributes i.e., $\{{\color{blue}{\bigcirc}}, {\color{red}\bigcirc}, {\color{blue}\times}, {\color{red}\times}\}$ with means at $ \mu_1 =(1,1,0)$, $\mu_2 =(2,2,0)$, $\mu_3=(2, 2.5,0)$, $\mu_4=(2.5, 3,0)$ and identical covariance matrices $ \Sigma= \mathrm{diag}\,\big(0.3^2, 0.3^2, 0.3^2\big)$. The shape attribute is our target while color is the sensitive attribute as illustrated in Figure~\ref{fig:mixdata}. The goal of the ARL problem is to learn an encoder that projects the data such that it remains separable with respect to the shape and non-separable with respect to the color attribute.

We sample 4000 points to learn linear and non-linear (Gaussian kernel) encoders across $\lambda \in [0,1]$. To train the encoder, the one-hot encoding of target and sensitive labels are treated as the regression targets. Then, we freeze the encoder and train logistic regressors for the adversary and target task for each $\lambda$. We evaluate their classification performance on a separate set of 1000 samples. The resulting trade-off front between target and adversary performance is shown in Figure~\ref{fig-mog}. We make the following observations, (1) For $\lambda=1$, all methods achieve an accuracy of $50 \%$ for the adversary which indicates complete removal of features corresponding to the sensitive attribute via our encoding, (2) At small values of $\lambda$ the objective of Linear-ARL is close to being convex, hence the similarity in the trade-off fronts of Linear-SARL and SSGD in that region. However, everywhere else due to the iterative nature of SSGD, it is unable to find the global solution and achieve the same trade-off as Linear-SARL. (3) The non-linear encoder in the Kernel-SARL solution significantly outperforms both Linear-SARL and SSGD. The non-linear nature of the encoder enables it to strongly entangle the color attribute (50\% accuracy) while simultaneously achieving a higher target accuracy than the linear encoder. Figure~\ref{fig:gaussian} visualizes the learned embedding space $\bm{z}$ for different trade-offs between the target and adversary objectives.

Figure~\ref{fig:bounds-gauss} shows the mean squared error (MSE) of the adversary as we vary the relative trade-off $\lambda$ between the target and adversary objectives. The plot illustrates, (1) the lower and upper bounds $\alpha_{\min}$ and $\alpha_{\max}$ respectively calculated on the training dataset, (2) achievable adversary MSE computed on the training set $\alpha_{\mathrm{train}}$, and finally (3) achievable adversary MSE computed on the test set $\alpha_{\mathrm{test}}$. Observe that on the training dataset, all values of $\alpha \in [\alpha_{\min},\alpha_{\max}]$ are reachable as we sweep through $\lambda\in[0,1]$. This is however not the case on the test set as the bounds are computed through empirical moments as opposed to the true covariance matrices.

\subsection{Fair Classification}
We consider the task of learning representations that are invariant to a sensitive attribute on two datasets, Adult and German, from the UCI ML-repository~\cite{Dua:2019}. For comparison, apart from the raw features $\bX$, we consider several baselines that use DNNs and trained through simultaneous SGD; LFR~\cite{zemel2013learning}, VAE~\cite{kingma2013auto}, VFAE~\cite{louizos2015variational}, ML-ARL~\cite{xie2017controllable} and MaxEnt-ARL~\cite{roy2019mitigating}.

The Adult dataset contains $14$ attributes. There are $30,163$ and $15,060$ instances in the training and test sets, respectively. The target task is binary classification of annual income i.e., more or less than $50$K and the sensitive attribute is gender. Similarly, the German dataset contains $1000$ instances of individuals with $20$ different attributes. The target is to classify the credit of individuals as good or bad with the sensitive attribute being age.

\begin{table}[t]
    \centering
    \vspace{-2mm}
    \caption{Fair Classification Performance (in \%)
   \label{tab:uci}}
    \scalebox{0.7}{
    \begin{tabular}{l|ccc|ccc}
        \toprule
        & \multicolumn{3}{c}{Adult Dataset} & \multicolumn{3}{|c}{German Dataset}\\
        \midrule
        {Method} & Target & Sensitive &$\Delta^*$
        &Target & Sensitive & $\Delta^*$ \\
        & (income) & (gender) & & (credit) & (age) & \\
        \midrule
         Raw Data & 85.0 & 85.0 &17.6 & 80.0 & 87.0&6.0\\
         \midrule
         LFR~\cite{zemel2013learning} & 82.3 & 67.0& 0.4 & 72.3  & 80.5&0.5\\
         VAE~\cite{kingma2013auto} & 81.9 & 66.0 & 1.4& 72.5 & 79.5&1.5\\
         VFAE~\cite{louizos2015variational} & 81.3 & 67.0 & 0.4& 72.7 & 79.7&1.3\\
         ML-ARL~\cite{xie2017controllable} & 84.4 & 67.7&0.3 & 74.4& 80.2&0.8\\
         MaxEnt-ARL~\cite{roy2019mitigating} & 84.6 & 65.5 & 1.9 & 72.5 & 80.0 & 1.0\\
         \midrule
         Linear-SARL & {84.1} & {67.4} & {0.0}& {76.3} & {80.9} & {0.1}\\
         Kernel-SARL & {84.1} & {67.4} &{0.0} & {76.3} & {80.9} & {0.1}\\
        \bottomrule
            \multicolumn{6}{c}{\small{$^*$ Absolute difference between adversary accuracy and random guess}}\\
    \end{tabular}}
\end{table}
\begin{figure}[t]
    \centering
    \begin{tikzpicture}[scale=1.1]
    \begin{axis}[
    xticklabels={,,},ytick={0, 0.1,0.2,0.3,0.4,0.5,0.6,0.7,0.8,0.9},
        ybar,
        ymin=0,
        ymax=0.9,
        bar width=11pt,
        ymajorgrids=true,
        grid style=dashed
    ]
    \addplot [black,fill=cyan] table {figures/adult_weights.txt};
    \end{axis}
    \node[rotate=90] at (0.6, .45){ age};
    \node[rotate=90] at (1, .9){work class};
    \node[rotate=90] at (1.45, 1.2){final weight};
    \node[rotate=90] at (1.85, 0.8){education};
    \node[rotate=90] at (2.3, 1.1){education num};
    \node[rotate=90] at (2.75, 1.1){marital status};
    \node[rotate=90] at (3.2, 1){occupation};
    \node[rotate=90] at (3.65, 1){relationship};
    \node[rotate=90] at (4.1, 0.35){race};
    \node[rotate=90, red] at (4.53, 0.7){gender};
    \node[rotate=90] at (4.97, 0.95){capital gain};
    \node[rotate=90] at (5.4, 0.95){capital loss};
    \node[rotate=90] at (5.85, 0.85){hours/week};
    \node[rotate=90] at (6.3, 0.7){country};
    \end{tikzpicture}
    \caption{\textbf{Adult Dataset:} Magnitude of learned encoder weights $\bm{\Theta}_E$ for each semantic input feature.\label{fig:encoder}}
\end{figure}
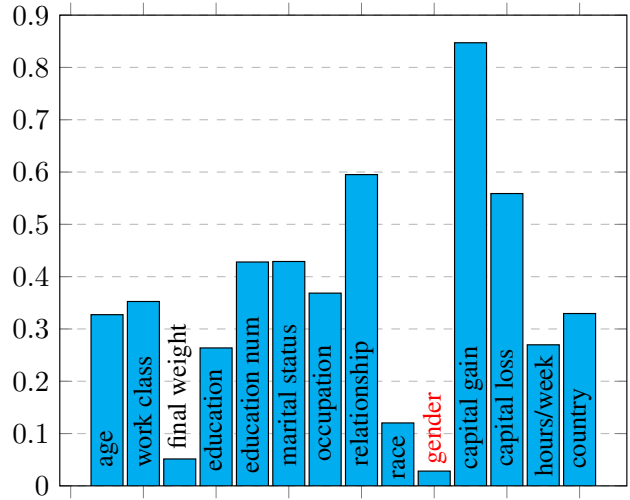

We learn encoders on the training set, after which, following the baselines, we freeze the encoder and train the target (logistic regression) and adversary (2 layer network with 64 units) classifiers on the training set.  Table~\ref{tab:uci} shows the performance of target and adversary on both datasets. Both Linear-SARL and Kernel-SARL outperform all DNN based baselines. For either of these tasks, the Kernel-SARL does not afford any additional benefit over Linear-SARL. For the adult dataset, the linear encoder maps the 14 input features to just one dimension. The weights assigned to each feature is shown in Figure~\ref{fig:encoder}. Notice that the encoder assigns almost zero weight to the gender feature in order to be fair with respect to the gender attribute.

\subsection{Illumination Invariant Face Classification}
\begin{table}[t]
    \centering
    \caption{Extended Yale B Performance (in \%) \label{table:yaleb}}
    \scalebox{0.7}{
    \begin{tabular}{l|cc|cc}
        \toprule
        {Method} & Adversary & Target & Adversary & Target \\
        & (illumination) & (identity)  & (identity) & (illumination) \\
        \midrule
         Raw Data & 96 & 78 & - &-\\
         \midrule
         VFAE~\cite{louizos2015variational} & 57 & 85 & - & - \\
         ML-ARL~\cite{xie2017controllable} & 57 & 89 & - & -\\
         MaxEnt-ARL~\cite{roy2019mitigating} & 40 & 89 & - & -\\
         \midrule
         Linear-SARL  & 21 & 81 & 3& 94\\
         Linear-SARL [EX] & 20 & 86 & 3& 97 \\
         Kernel-SARL & 20 & 86 & 3 & 96\\
         Kernel-SARL [EX]& 20 & 88 & 3 & 96\\
        \bottomrule
    \end{tabular}}
\end{table}
This task pertains to face classification under different illumination conditions on the Extended Yale B dataset~\cite{georghiades2001few}. It comprises of face images of 38 people under five different light source directions, namely, upper right, lower right, lower left, upper left, and front. The target task is to establish the identity of the person in the image with the direction of the light being the sensitive attribute. Since the direction of lighting is independent of identity, the ideal ARL solution should obtain a representation $\bm{z}$ that is devoid of any sensitive information. We first followed the experimental setup of Xie \etal~\cite{xie2017controllable} in terms of the train/test split strategy i.e., 190 samples (5 from each class) for training and 1096 images for testing. Our global solution was able to completely remove illumination from the embedding resulting in the adversary accuracy being 20\% i.e., random chance. To investigate further, we consider different variations of this problem, flipping target and sensitive attributes and exchanging training and test sets. The complete set of results, including DNN based baselines are reported in Table~\ref{table:yaleb} ([EX] corresponds to exchanging training and testing sets). In all these cases, our solution was able to completely remove the sensitive features resulting in adversary performance that is no better than random chance. Simultaneously, the embedding is also competitive with the baselines on the target task.

\subsection{CIFAR-100}
The CIFAR-100 dataset \cite{krizhevsky2009learning} consists of 50,000 images from 100 classes that are further grouped into 20 superclasses. Each image is therefore associated with two attributes, a ``fine'' class label and a ``coarse" superclass label. We consider a setup where the ``coarse" and ``fine" labels are the target and sensitive attributes, respectively. For Linear-SARL and Kernel-SARL (degree five polynomial kernel) and SSGD we use features (64-dimensional) extracted from a pre-trained ResNet-110 model as an input to the encoder, instead of raw images. From these features, the encoder is tasked with aiding the target predictor and hindering the adversary. This setup serves as an example to illustrate how invariance can be ``imparted" to an existing biased pre-trained representation. We also consider two DNN baselines, ML-ARL~\cite{xie2017controllable} and MaxEnt-ARL~\cite{roy2019mitigating}. Unlike our scenario, where the pre-trained layers of ResNet-18 are not adapted, the baselines optimize the entire encoder for the ARL task. For evaluation, once the encoder is learned and frozen, we train a discriminator and adversary as 2-layer networks with 64 neurons each. Therefore, although our approach uses linear regressor as adversary at training, we evaluate against stronger adversaries at test time. In contrast, the baselines train and evaluate against adversaries with equal capacity.

Figure~\ref{fig-cifar} shows the trade-off in accuracy between the target predictor and adversary. We observe that, (1) Kernel-ARL significantly outperforms Linear-SARL. Since the former implicitly maps the data into an higher dimensional space, the sensitive features are potentially disentangled sufficiently for the linear encoder in that space to discard such information. Therefore, even for large values of $\lambda$, Kernel-SARL is able to simultaneously achieve high target accuracy while keeping the adversary performance low. (2) Despite being handicapped by the fact that Kernel-SARL is evaluated against stronger adversaries than it is trained against, its performance is comparable to that of the DNN baselines. In fact, it outperforms both ML-ARL and MaxEnt-ARL with respect to the target task. (3) Despite repeated attempts with different hyper-parameters and choice of optimizers, SSGD was highly unstable across most datasets and often got stuck in a local optima and failed to find good solutions.
\begin{figure}[t]
  \centering
  \includegraphics[width=0.45\textwidth]{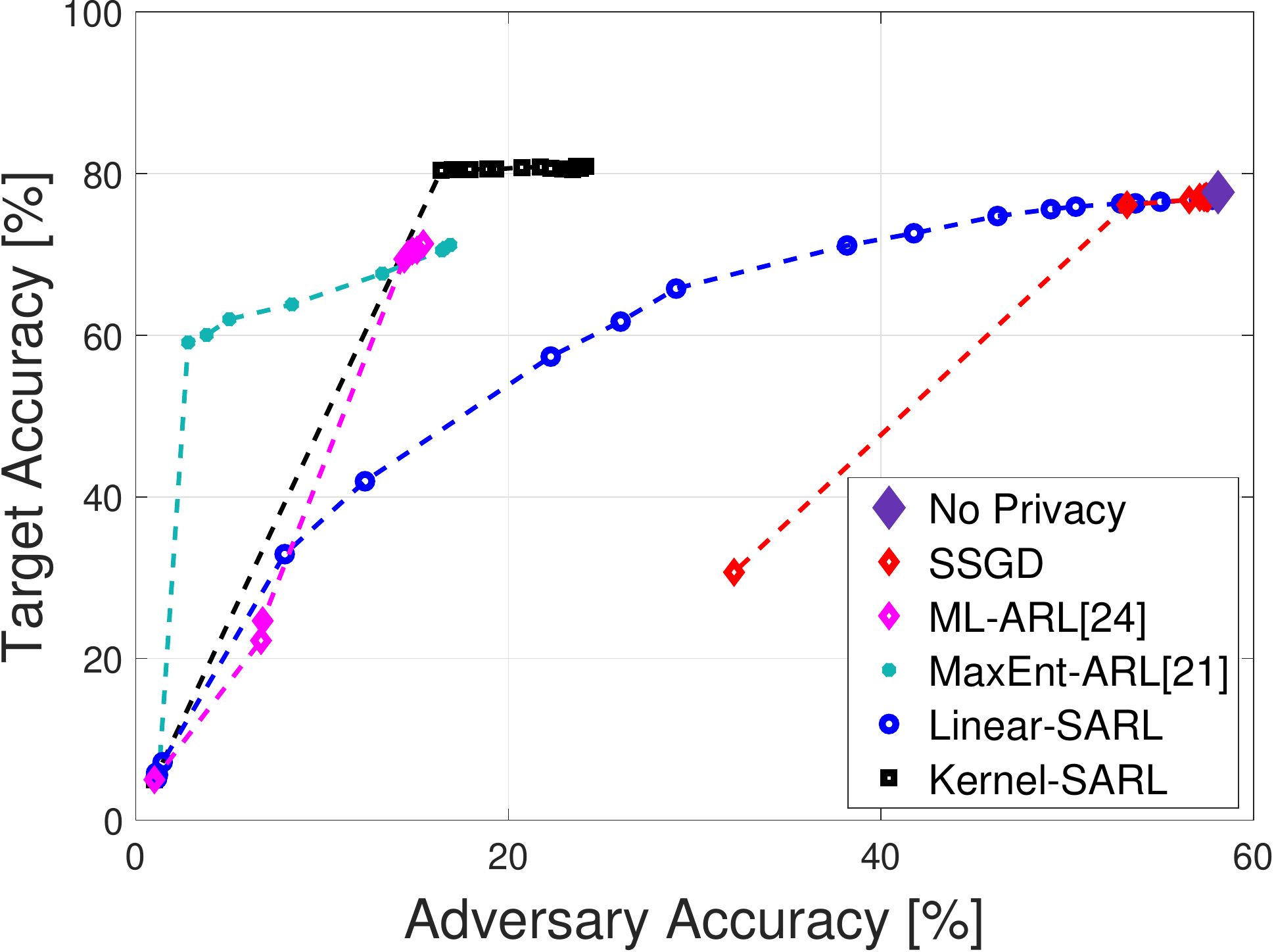}
  \caption{\textbf{CIFAR-100:} Trade-off between target performance and leakage of sensitive attribute by adversary.\label{fig-cifar}}
\end{figure}

Figure~\ref{fig:bounds-cifar} shows the mean squared error (MSE) of the adversary as we vary the relative trade-off $\lambda$ between the target and adversary objectives. The plot illustrates, (1) the lower and upper bounds $\alpha_{\min}$ and $\alpha_{\max}$ respectively calculated on the training dataset, (2) achievable adversary MSE computed on the training set $\alpha_{\mathrm{train}}$, and finally (3) achievable adversary MSE computed on the test set $\alpha_{\mathrm{test}}$. Observe that on the training dataset, all values of $\alpha \in [\alpha_{\min},\alpha_{\max}]$ are reachable as we sweep through $\lambda\in[0,1]$. This is however not the case on the test set as the bounds are computed through empirical moments as opposed to the true covariance matrices.

\begin{figure}[t]
    \centering
\begin{subfigure}[b]{0.23\textwidth}
    \centering
    \includegraphics[width=\textwidth, height=31mm]{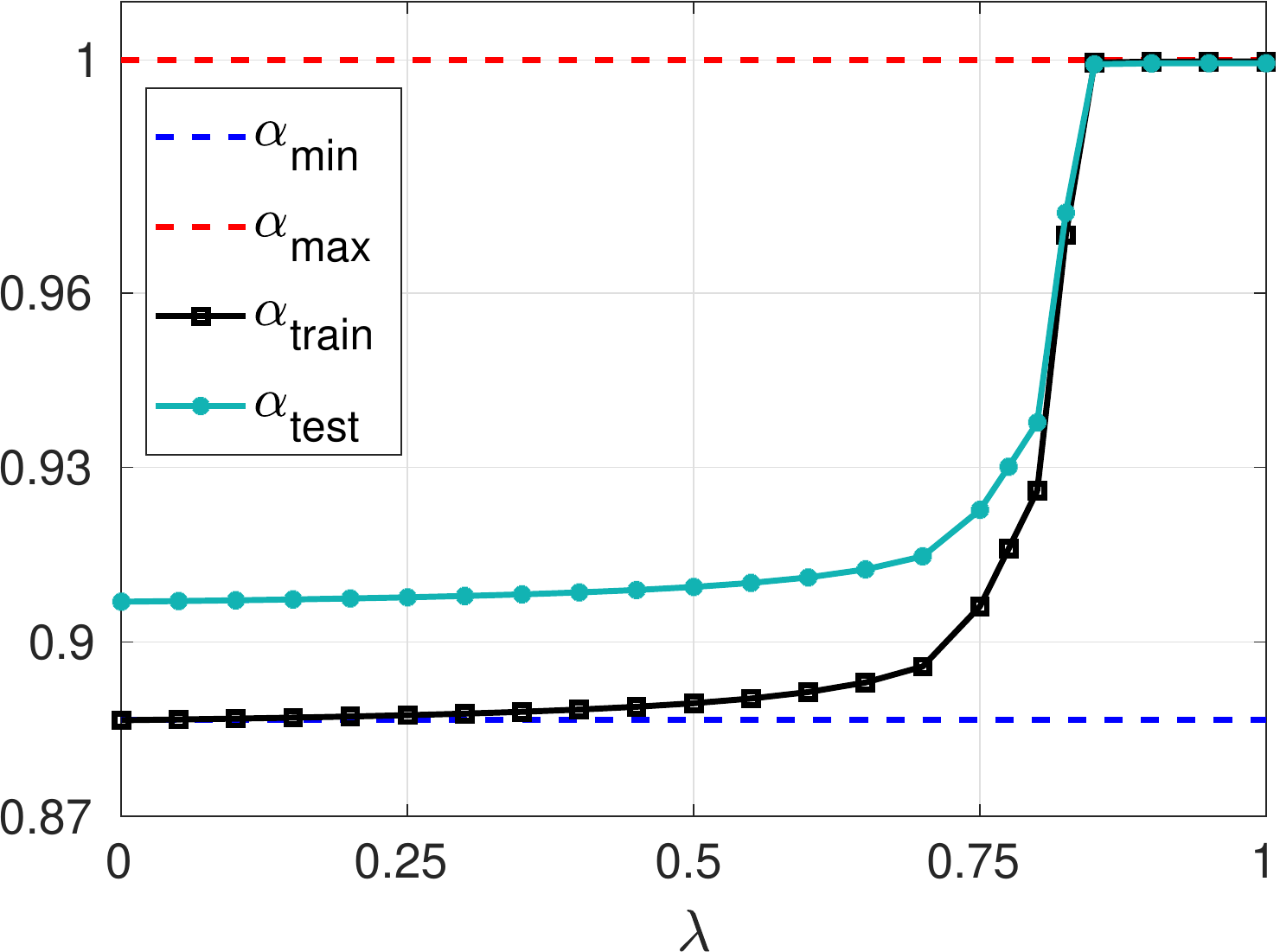}
    \caption{Linear-SARL \label{fig:bounds-cifar-linear}}
\end{subfigure}
\begin{subfigure}[b]{0.23\textwidth}
    \centering
    \includegraphics[width=\textwidth, height=31mm]{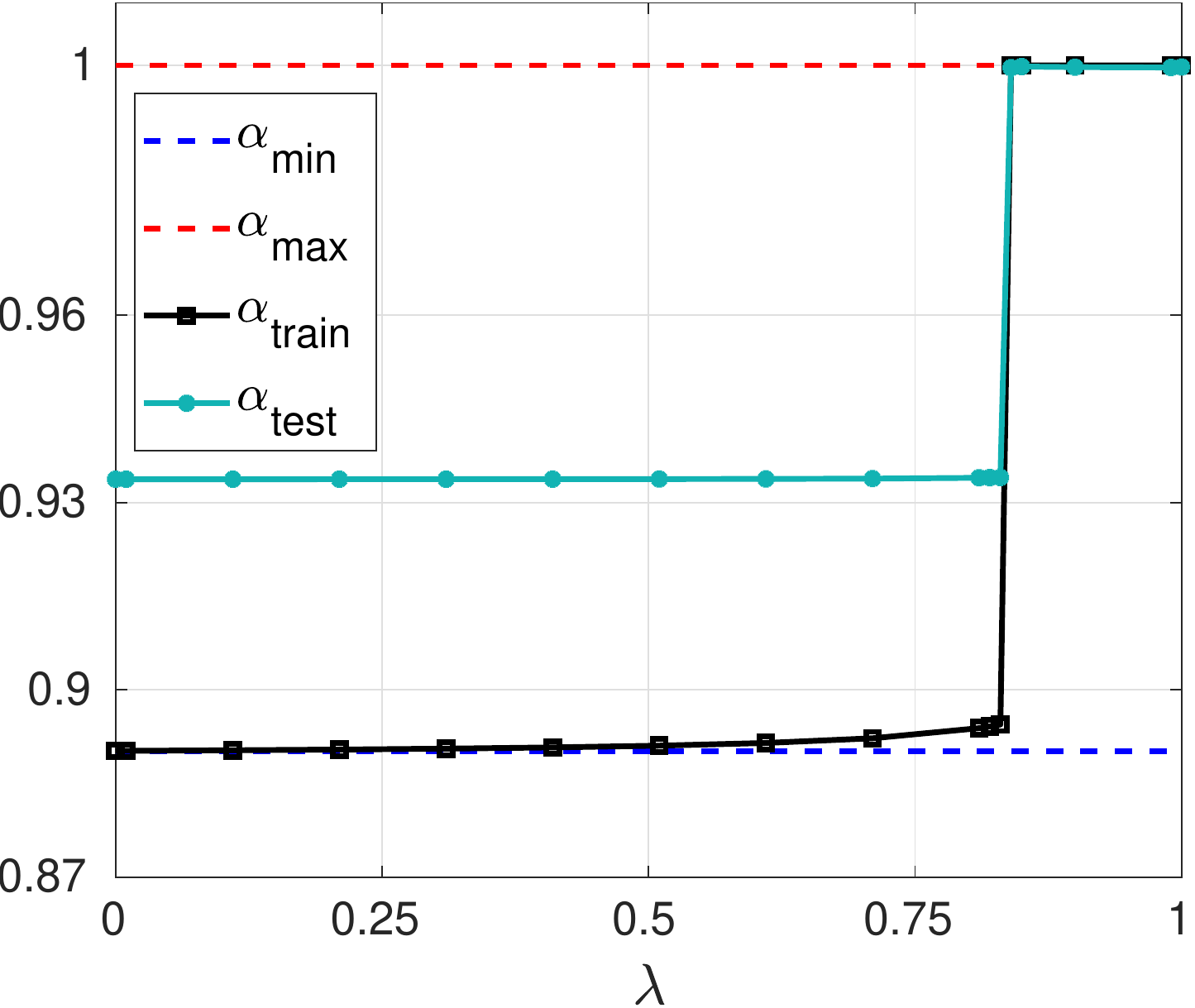}
    \caption{Kernel-SARL\label{fig:bounds-cifar-kernel}}
\end{subfigure}
\caption{\textbf{CIFAR-100:} Lower and upper bounds on adversary loss, $\alpha_{min}$ and $\alpha_{max}$, computed on training set. The loss achieved by our solution as we vary $\lambda$ is shown on the training and testing sets, $\alpha_{train}$ and $\alpha_{test}$, respectively.\label{fig:bounds-cifar}}
\end{figure}
\begin{figure}[t]
    \centering
    \includegraphics[width=0.45\textwidth]{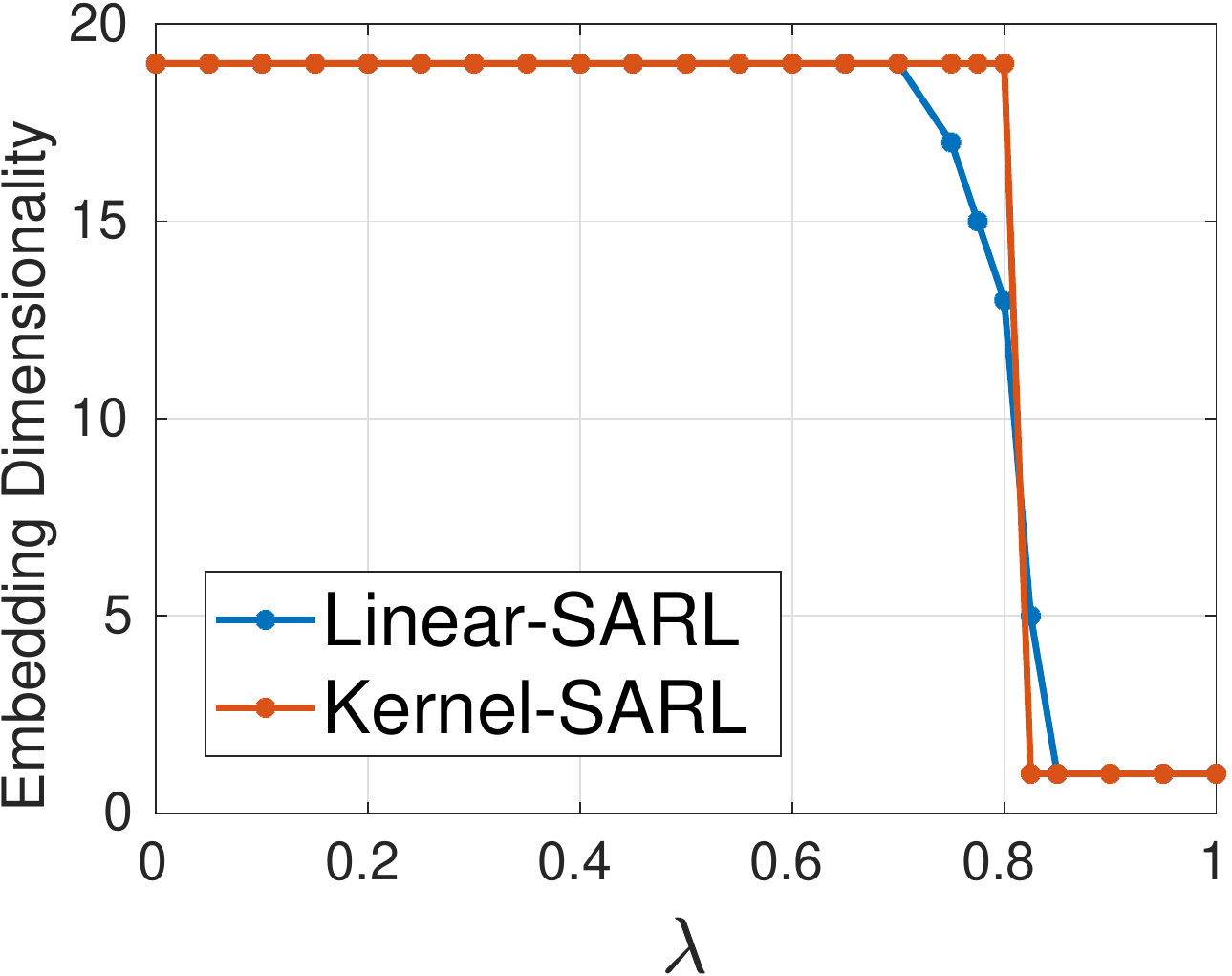}
    \caption{\textbf{CIFAR-100:} Optimal embedding dimensionality learned by SARL. At small values of $\lambda$, the objective favors the target task which predicts 20 classes. Thus, embedding dimensionality of 19 is optimal for a linear target regressor. At large values of $\lambda$, the objective only seeks to hinder the adversary. Thus, SARL determines the optimal dimensionality of the embedding as one. \label{fig:cifar-r}}
\label{fig:cifar-r}
\end{figure}

Figure~\ref{fig:cifar-r} plots the optimal embedding dimensionality provided by SARL as a function of the trade-off parameter $\lambda$. At small values of $\lambda$, the objective favors the target task i.e., 20 class prediction. Thus, SARL does indeed determine the optimal dimensionality of 19 for a 20 class linear target regressor. However, at large values of $\lambda$, the objective only seeks to hinder the sensitive task i.e., 100 class prediction. In this case, the ideal embedding dimensionality from the perspective of the linear adversary regressor is at least 99. The SARL ascertained dimensionality of one is, thus, optimal for maximally mitigating the leakage of sensitive attribute from the embedding. However, unsurprisingly, the target task also suffers significantly.

%% file: conclusion.tex
% conclusion.tex

\section{Concluding Remarks}

We studied the ``linear" form of adversarial representation learning (ARL), where all the entities are linear functions. We showed that the optimization problem even for this simplified version is both non-convex and non-differentiable. Using tools from spectral learning we obtained a closed form expression for the global optima and derived analytical bounds on the achievable utility and invariance. We also extended these results to non-linear parameterizations through kernelization. Numerical experiments on multiple datasets indicated that the global optima solution of the ``kernel" form of ARL is able to obtain a trade-off between utility and invariance that is comparable to that of local optima solutions of deep neural network based ARL. At the same time, unlike DNN based solutions, the proposed method can, (1) analytically determine the achievable utility and invariance bounds, and (2) provide explicit control over the trade-off between utility and invariance.

Admittedly, the results presented in this paper do not extend directly to deep neural network based formulations of ARL. However, we believe it sheds light on nature of the ARL optimization problem and aids our understanding of the ARL problem. It helps delineate the role of the optimization algorithm and the choice of embedding function, highlighting the trade-off between the expressivity of the functions and our ability to obtain the global optima of the adversarial game. We consider our contribution as the first step towards controlling the non-convexity that naturally appears in game-theoretic representation learning.

\vspace{5pt}
\noindent\textbf{Acknowledgements:} This work was performed under the following financial assistance award 60NANB18D210 from U.S. Department of Commerce, National Institute of Standards and Technology.

%% file: appendix.tex
% appendix.tex

Here we include; (a) Section \ref{sec:lemma1}: Proof of Lemma \ref{th1}, (b) Section \ref{sec:relation}: Proof of relation between constrained optimization problem in~(\ref{eq-constrained}) and its Lagrangian formulation in~(\ref{eq-main}), (c) Section \ref{sec:theorem2}: Proof of Theorem \ref{th2}, (d) Section \ref{sec:theorem3}: Proof of Theorem \ref{th3}, (e) Section \ref{sec:linear}: Empirical moments based solution to linear encoder, (f) Section \ref{sec:kernel}: A detailed description of the Kernel-ARL extension, including derivation of its solution, and (g) Section \ref{sec:lemma4}: Proof of Lemma \ref{th4}.

\section{Proofs}

We recall that for any square matrix $\bM$, its trace, denoted by  $\tr[\bM]$, is defined as the sum of all its diagonal elements. The Frobenius norm of $\bM$ can be obtained as $\|\bM\|_F^2 = \tr (\bM \bM^T)$. This allows us to express the MSE of a centered random vector in terms of its covariance matrix:
\begin{equation}   \label{eq-x1}
\E \Big\{ \big\|\y -\bb_y \big\|^2 \Big\} =\tr  \Big[ \E\big\{(\y -\bb_y ) (\y -\bb_y )^T  \big\} \Big]  = \tr[\bC_y]. \nn
\end{equation}
Let $\mathbf{A}$ and $\mathbf{B}$ be two arbitrary matrices with the same dimension. Further, assume that the subspace $\mathcal R (\mathbf{A})$
is orthogonal to $\mathcal R (\mathbf{B})$. Then, using orthogonal decomposition (i.e., Pythagoras theorem), we have
\[
\big\| \mathbf{A} + \mathbf{B} \big\|_F^2 = \big\| \mathbf{A} \big\|_F^2 + \big\|\mathbf{B} \big\|_F^2.
\]

\noindent We provide the statements of the lemmas and theorems for sake of convenience, along with their proofs.

\subsection{Proof of Lemma~1 \label{sec:lemma1}}
\begin{lemma}
Let $\x$ and $\bm{t}$ be two random vectors with $\E[\x]=0$,  $\E[\bm{t}]=\bb$, and $\bC_x\succ0$. Consider a linear regressor, $\bm{\hat{t}} = \bW \z + \bb$, where $\bW \in \R^{m \times r}$ is the parameter matrix, and $\z\in\R^r$ is an encoded version of $\x$ for a given $\bt_E$: $\x  \mapsto \z =\bt_E \x, \quad \bt_E \in \R^{r\times d}$. The  minimum MSE that can be achieved by designing $\bW$ is given as
\[\min_{\bW} \E[\|\bm{t} - \bm{\hat{t}}\|^2 ]  = \tr \big[\bC_t \big] - \big\|P_\M \bQ_x^{-T} \bC_{xt} \big\|_F^2
\]
where $\bM = \bQ_x \bt_E^T \in \R^{d\times r}$, and $ \bQ_x \in \R^{d\times d}$ is a Cholesky factor of $\bC_x$ as shown in (\ref{Cholesky}).
\end{lemma}

\begin{proof}
Direct calculation yields:
\begin{equation}
  \label{eq-lem-1}
  \begin{aligned}
    J_t &= \E \Big\{ \big\|\bm{t} -\hat{\bm{t}} \big\|^2 \Big\} \nn \\
    &= \tr \Big[ \E\Big\{ (\bm{t} -\bb - \bW \z ) (\bm{t} -\bb -\bW \z )^T   \Big\} \Big] \nn \\
    &= \tr \Big[ \E\Big\{ (\bm{t} -\bb) (\bm{t} -\bb)^T +(\bW \bt_E \x ) (\bW \bt_E \x )^T     - (\bm{t} -\bb)(\bW \bt_E \x )^T - (\bW \bt_E \x ) (\bm{t} -\bb )^T \Big\} \Big] \\
    &= \tr \Big[\bC_t + (\bW \bt_E) \bC_x  (\bW \bt_E)^T   - \bC_{tx} (\bW \bt_E )^T - (\bW \bt_E) \bC_{tx}^T    \Big] \nn \\
    &= \tr \Big[\bC_t + (\bW \bt_E  \bQ_x^T) (\bW \bt_E \bQ_x^T)^T   - \bC_{tx} (\bW \bt_E )^T - (\bW \bt_E) \bC_{tx}^T   \Big] \nn \\
    &= \tr \Big[ (\bW \bt_E  \bQ_x^T - \bC_{tx}\bQ_x^{-1}) ( \bW \bt_E  \bQ_x^T - \bC_{tx}\bQ_x^{-1} )^T   +\bC_t-  (\bC_{tx}\bQ_x^{-1}) (\bC_{tx}\bQ_x^{-1})^T  \Big] \nn \\
    &= \big\| \bQ_x \bt_E^T \bW^T  - \bQ_x^{-T}  \bC_{xy}\big\|_F^2 - \big\|\bQ_x^{-T}  \bC_{xt} \big\|_F^2 + \tr [\bC_t]
\end{aligned}
\end{equation}
Hence, the minimizer of $J_t$ is obtained by minimizing the first term in the last equation, which is a standard least square error problem. Let $\bM=\bQ_x \bt_E^T$, then the minimizer is given by
\[\bW^T  = \bM^\dagger \bQ_x^{-T} \bC_{xt}\]

Using the orthogonal decomposition
\[
\big \| \bQ_x^{-T}\bC_{xt}  \big\|_F^2 = \big\|P_{\M} \bQ_x^{-T}\bC_{xt} \big\|_F^2+\big\|P_{\M^\perp}\bQ_x^{-T}\bC_{xt} \big\|_F^2
\]
and
\begin{equation}
  \begin{aligned}
    \big\| \bQ_x \bt_E^T \bW^T  - \bQ_x^{-T}  \bC_{xt}\big\|_F^2 &= \big\| \bM \bW^T  - P_\M\bQ_x^{-T} \bC_{xt}\big\|_F^2+\big\|P_{\M^\perp}\bQ_x^{-T}\bC_{xt} \big\|_F^2 \\
    &= \big\| \underbrace{\bM \bM^\dagger}_{P_\M} \bQ_x^{-T} \bC_{xt}  - P_\M\bQ_x^{-T} \bC_{xt}\big\|_F^2+\big\|P_{\M^\perp}\bQ_x^{-T}\bC_{xt} \big\|_F^2\\
    &= \big\|P_{\M^\perp}\bQ_x^{-T}\bC_{xt} \big\|_F^2 \nn
\end{aligned}
\end{equation}
Therefore, we obtain the minimum value as,
\[
\tr\big [\bC_t \big] - \big\|P_{\M}\bQ_x^{-T} \bC_{xt}\big\|_F^2
\]
\end{proof}

\subsection{Relation Between Constrained Optimization Problem in~(\ref{eq-constrained}) and its Lagrangian Formulation in~(\ref{eq-main}) \label{sec:relation}}
\noindent Consider the optimization problem in~(\ref{eq-constrained})
\begin{equation}\label{pro1}
\bG_\alpha =\arg \min_\bG J_y(\bG), \quad \mathrm{s.t.}  \quad  J_s(\bG) \geq \alpha .
\end{equation}
and the optimization problem in~(\ref{eq-main})
\begin{equation}\label{pro2}
\bG_\lambda = \arg \min_\bG J_\lambda (\bG)
\end{equation}
where
\[ J_\lambda (\bG) = (1 -\lambda) J_y (\bG) - \lambda J_s (\bG), \quad  \lambda \in [0,1]
\]

\noindent\textbf{Claim } For each $\lambda \in [0,1)$, solution $\bG_\lambda$ of (\ref{pro2}) is also a solution of (\ref{pro1}) with
\begin{equation}\label{eq-alf1}
\alpha=J_s(\bG_\lambda).
\end{equation}

\begin{proof}
Let us consider~(\ref{pro1}) while assuming that~(\ref{pro2}) is satisfied. For each $\lambda$ and corresponding solution $\bG_\lambda$,  let $\alpha$ be given as in (\ref{eq-alf1}). For an arbitrary $\bG$ satisfying $J_s(\bG)\ge\alpha$, we have
\begin{equation}
  \begin{aligned}
    (1-\lambda)J_y(\bG_\lambda)-\lambda \alpha &=  (1-\lambda)J_y(\bG_\lambda)-\lambda J_s(\bG_\lambda)\\
    &\le (1-\lambda)J_y(\bG)-\lambda J_s(\bG),
\end{aligned}
\end{equation}
where the second step is from the assumption that~(\ref{pro2}) is satisfied. Consequently, we have,
\begin{equation}
  (1-\lambda)\big[J_y(\bG)-J_y(\bG_\lambda)\big] \ge \lambda\big[J_s(\bG)-\alpha\big]\ge 0
\end{equation}
Since $J_s(\bG)\ge\alpha$, this implies that $J_y(\bG)\ge J_y(\bG_\lambda)$ and consequently $\bG_\lambda$ is a possible minimizer of problem~(\ref{pro1}).
\end{proof}

\subsection{Proof of Theorem~2\label{sec:theorem2}}
\begin{theorem}
As a function of $\bG_E \in \R^{d\times r}$, the objective function in equation~(\ref{eq-main}) is neither convex nor differentiable.
\end{theorem}
\begin{proof}
Recall that $P_\G$ is equal to $\bG_E (\bG_E^T \bG_E)^\dagger \bG_E^T$. Therefore, due to the involvement of the pseudo inverse,~(\ref{eq-main}) is not differentiable (see~\cite{golub1973differentiation}).

For non-convexity consider the theorem that $f(\bG_E)$ is convex in $\bG_E\in \R^{d\times r}$ if and only if $h(t)=f(t\,\bG_1+\bG_2)$ is convex  in $t\in \R$ for any constants $\bG_1,\, \bG_2\in\R^{d\times r}$ (see~\cite{boyd2004convex}).

In order to use the above theorem, consider rank one matrices
\[
\bG_1 = \begin{bmatrix}
    1       & 0  & \dots & 0 \\
    0       &  0 & \dots & 0 \\
        0       &  0 & \dots & 0 \\
    \vdots  & \vdots & \ddots &\\
    0       & 0& \dots & 0
\end{bmatrix}
\quad \mathrm{and} \quad
\bG_2 = \begin{bmatrix}
    1       & 0  & \dots & 0 \\
    1      &  0 & \dots & 0 \\
        0      &  0 & \dots & 0 \\
    \vdots  & \vdots & \ddots &\\
    0       & 0& \dots & 0
\end{bmatrix}.
\]
Define $\bG_E = (t\,\bG_1+\bG_2)$. Then
\[
P_\G (t) = \bG_E (\bG_E^T \bG_E)^\dagger \bG_E^T
= \frac{1}{(t+1)^2+1}
\begin{bmatrix}
    (t+1)^2       & (t+1)  & 0 & \dots & 0 \\
    (t+1)      &  1 & 0 & \dots & 0 \\
    0      &  0 & 0 & \dots & 0 \\
    \vdots  & \vdots & \vdots &  \ddots &\\
    0       & 0 & 0 & \dots & 0
\end{bmatrix}.
\]
Using basic properties of trace we get,
\begin{equation}
(1-\lambda)J_y(\bG_E)-\lambda J_s(\bG_E) =  \tr \big[ P_\G (t)\bB\big],\nn
\end{equation}
where the matrix $\bB$ is given in~(\ref{eq-B}) and we used Lemma \ref{th1}. Now, represent $\bB$ as
\[
\bB = \begin{bmatrix}
    b_{11}       & b_{12}  & \dots & b_{1d} \\
    b_{12}      &  b_{22}  & \dots & b_{2d}  \\
    \vdots  & \vdots & \ddots &\\
    b_{1d}        & b_{2d} & \dots & b_{dd}
\end{bmatrix}.
\]
Thus,
\begin{equation}
  \tr\big[ P_\G (t)\bB\big] = b_{11}+ \frac{2b_{12}(t+1)+b_{22}-b_{11}}{(t+1)^2+1} \nn
\end{equation}
It can be shown that the above function of $t$ is convex only if $b_{12}=0$ and $b_{11}=b_{22}$. On the other hand, if these two conditions hold, it can be similarly shown that  $(1-\lambda)J_y(\bG_E)-\lambda J_s(\bG_E)$ is non-convex by considering a different pair of matrices $\bG_1$ and $\bG_2$. This implies that $(1-\lambda)J_y(\bG_E)-\lambda J_s(\bG_E)$ is not convex.
\end{proof}

\subsection{Proof of Theorem~3\label{sec:theorem3}}
\begin{theorem}
Assume that the number of negative eigenvalues ($\beta$) of $\bB$ in~(\ref{eq-B}) is $j$. Denote $\gamma=\min\{r, j \}$. Then, the minimum value in~(\ref{eq-main2}) is given as,
\begin{equation}
\beta_1 + \beta_{2}+\cdots +\beta_{\gamma} \nonumber
\end{equation}
where $\beta_1 \leq \beta_2 \leq \ldots \leq \beta_{\gamma} < 0$ are the $\gamma$ smallest eigenvalues of $\bB$. And the minimum can be attained by $\bG_E =\bV$, where the columns of $\bV$ are  eigenvectors corresponding to all the $\gamma$ negative eigenvalues of $\bB$.
\end{theorem}

\begin{proof}
Consider the inner optimization problem of~(\ref{eq-main3}) in~(\ref{eq-main2}). Using the trace optimization problems and their solutions in~\cite{kokiopoulou2011trace}, we get
\begin{equation}
\min_{\bG_E^T \bG_E = \bI_i} J_\lambda(\bG_E) \,=
\min_{\bG_E^T \bG_E = \bI_i}\tr{\big[\bG_E^T \bB \bG_E\big]}\,=\beta_1 + \beta_{2}+\cdots +\beta_{i}\nn,
\end{equation}
where $\beta_1,\, \beta_{2},\,\dots,\, \beta_{i}$ are $i$ smallest eigenvalues of $\bB$ and minimum value can be achieved by the matrix $\bV$ whose columns are corresponding eigenvectors.
If the number of negative eigenvalues of $\bB$ is less than $r$, then the optimum $i$ in~(\ref{eq-main2}) is $j$, otherwise the optimum $i$ is $r$.
\end{proof}

\section{Empirical Moments Based Solution to Linear Encoder\label{sec:linear}}
In many practical scenarios, we only have access to data samples but not to the true mean vectors and covariance matrices. Therefore, the solution in Section \ref{sec-linear} might not be feasible in such as case. In this Section, we provide an approach to solve the optimization problem in Section \ref{sec-linear} which relies on empirical moments and is valid even if the covariance matrix $\bC_x$ is not full-rank.

Firstly, for a given $\bt_E$, we find
\[
\J_y = \min_{\bW_y,\bb_y} \mse\,(\hat \y - \y).
\]
Note that the above optimization problem can be separated over $\bW_y$, $\bb_y$. Therefore, for a given $\bW_y$, we first minimize over $\bb_y$:
\begin{equation}
  \begin{aligned}
    &\min_{\bb_y}\E\Big\{\big\|\bW_y \bt_E \x+\bb_y-\y\big\|^2\Big\}\\
    &=\min_{\bb_y}\frac{1}{n}\sum_{k=1}^n \big\|\bW_y \bt_E \x_k+\bb_y-\y_k\big\|^2\\
    &=\frac{1}{n}\sum_{k=1}^n \big\|\bW_y \bt_E \x_k+\bc-\y_k\big\|^2\nn
  \end{aligned}
\end{equation}
where we used empirical expectation in the second stage and the minimizer $\bc$ is
\begin{equation}
  \label{c-linear}
  \begin{aligned}
    \bc &= \frac{1}{n} \sum_{k=1}^n \Big(\y_k - \bW_y \bt_E \x_k\Big)\\
    &= \frac{1}{n} \sum_{k=1}^n \y_k -  \bW_y \bt_E  \frac{1}{n} \sum_{k=1}^n  \x_k\\
    &= \E\big\{ \y \big\} - \bW_y \bt_E \,\E\big\{ \x \big\}
\end{aligned}
\end{equation}

Let all the columns of matrix $\bC$ be equal to $\bc$. We now have,
\begin{equation}
  \begin{aligned}
    J_y &= \min_{\bW_y,\bb_y} \mse\,(\hat \y - \y)\nn\\
    &= \min_{\bW_y} \frac{1}{n}\big\| \bW_y \bt_E \bX + \bC-\bY \big\|_F^2\nn\\
    &= \min_{\bW_y} \frac{1}{n}\big\| \bW_y \bt_E\tilde{\bX} -\tilde{\bY} \big\|_F^2\nn\\
    &= \min_{\bW_y} \frac{1}{n}\big\| \tilde{\bX}^T \bt_E^T \bW_y^T - \tilde{\bY}^T \big\|_F^2\nn\\
    &= \min_{\bW_y} \frac{1}{n}\big\| \bM \bW_y^T- P_\M \tilde {\bY}^T  \big\|_F^2 + \frac{1}{n}\big\| P_{\M^\perp} \tilde{\bY}^T \big\|_F^2 \nn\\
    &= \frac{1}{n}\big\| \underbrace{\bM \bM^\dagger}_{P_\M} P_\M \tilde {\bY}^T- P_\M \tilde {\bY}^T  \big\|_F^2 + \frac{1}{n}\big\| P_{\M^\perp} \tilde{\bY}^T \big\|_F^2 \nn\\
    &= \frac{1}{n}\big\| P_{\M^\perp} \tilde{\bY}^T \big\|_F^2\nn\\
    &= \frac{1}{n}\big\| \tilde{\bY}^T \big\|_F^2-\frac{1}{n}\big\| P_\M \tilde{\bY}^T \big\|_F^2\nn
\end{aligned}
\end{equation}
where in the third step we used~(\ref{c-linear}), $\bM=\tilde{\bX} ^T \bt_E^T$ and the fifth step is due to orthogonal decomposition. Using the same approach, we get
\begin{equation}
  \label{alpha}
\J_s = \frac{1}{n}\big\| \tilde{\bS}^T \big\|_F^2 - \frac{1}{n}\big\| P_\M \tilde{\bS}^T \big\|_F^2
\end{equation}

Now, assume that the columns of $\bL_x$ are orthogonal basis for the column space of $\tilde{\bX}^T$. Therefore, for any $\bM$, there exists a $\bG_E$ such that $\bL_x \bG_E =\bM$. In general, there is no bijection between $\bt_E$ and $\bG_E$ in the equality $\tilde{\bX}^T \bt_E^T = \bL_x \bG_E$. But, there is a bijection between $\bG_E$ and $\bt_E$ when constrained to $\bt_E$'s in which $\mathcal R(\bt_E^T)\subseteq {\N(\tilde{\bX}^T)}^\perp$. This restricted bijection is sufficient to be considered, since for any $\bt_E^T\in \N(\tilde{\bX}^T)$ we have $\bM=\bf0$. Once $\bG_E$ is determined, $\bt_E^T$ can be obtained as,
\[
\bt_E^T =  (\tilde{\bX}^T)^\dagger \bL_x \bG_E + \bt_0, \ \ \bt_0\subseteq \N(\tilde{\bX}^T).
\]
However, since
\[
\big\|\bt_E\big\|_F^2=\big\|\bt_E^T\big\|_F^2 = \big\|(\tilde{\bX}^T)^\dagger \bL_x \bG_E\big\|_F^2 + \big\|\bt_0\big\|_F^2,
\]
choosing $\bt_0 = \bf 0$ results in minimum $\big\|\bt_E\big\|_F$, which is favorable in terms of robustness to noise.

By choosing $\bt_0 = \bf 0$, determining the encoder $\bt_E$ would be equivalent to determining $\bG_E$. Similar to~(\ref{A-G}), we have $P_\M = \bL_x P_\G \bL_x^T$. If we assume that the rank of $P_\G$ is $i$, $J_\lambda(\bG_E)$ in~(\ref{eq-main3}) can be expressed as,
\begin{equation}
J_\lambda (\bG_E) = \lambda \big\| \bL_x \bG_E \bG_E^T \bL_x^T \tilde{\bS}^T\big\|_F^2 - (1-\lambda) \big\| \bL_x \bG_E \bG_E^T \bL_x^T \tilde{\bY}^T\big\|_F^2\nn
\end{equation}
where $\bG_E \bG_E^T=P_\G $ for some orthogonal matrix $\bG_E\in \R^{d\times i}$. This resembles the optimization problem in (\ref{eq-main2}) and therefore it has the same solution as Theorem~\ref{th3} with modified $\bB$ given by
\begin{equation}
  \label{B-linear}
  \bB = \bL_x^T \Big(\lambda \tilde{\bS}^T \tilde{\bS} -(1-\lambda)\tilde{\bY}^T \tilde{\bY} \Big) \bL_x
\end{equation}
Once $\bG_E$ is determined, $\bt_E$ can be obtained as $\bG_E^T \bL_x^T (\tilde{\bX})^\dagger$. Algorithm~\ref{alg1} summarizes our entire solution for the constrained optimization problem in~(\ref{eq-constrained}) through the solution of the Lagrangian version in~(\ref{eq-main}).

\algdef{SE}[DOWHILE]{Do}{doWhile}{\algorithmicdo}[1]{\algorithmicwhile\ #1}
\begin{algorithm}[t]
    \caption{Spectral Adversarial Representation Learning\label{alg1}}
    \begin{algorithmic}[1] % The number tells where the line numbering should start
            \State \textbf{Input:} data $\bX$, target labels $\bY$, sensitive labels $\bS$, tolerable leakage $\alpha_\min\le\alpha_{\mathrm{tol}}\le\alpha_\max$, $\epsilon$
            \State \textbf{Output:} linear encoder parameters $\bt_E$
            \State $\bL_x \gets  \text{orthonormalize basis of }\tilde{\bX}^T$
             \State Initiate $\lambda = 1/2$, $\lambda_\min=0$ and $\lambda_\max=1$
            \Do
                \State Calculate $\bB$ in~(\ref{B-linear})
                \State $\bG_E \gets$ eigenvectors of negative eigenvalues of $\bB$
                \State $\bt_E \gets \bG_E^T \bL_x^T (\tilde{\bX})^\dagger$~
                \State Calculate $\alpha$ using~(\ref{alpha})
                \If{$\alpha<(\alpha_{\mathrm{tol}}-\epsilon)$} {$\lambda_\min=\lambda$ and $\lambda \gets  (\lambda+\lambda_\max)/2$}
                \ElsIf{$\alpha>(\alpha_{\mathrm{tol}}+\epsilon)$}  {$\lambda_\max = \lambda$ and $\lambda \gets  (\lambda+\lambda_\min)/2$}
                \EndIf
            \doWhile {$\big|\alpha - \alpha_{\mathrm{tol}}\big|\ge \epsilon$}
    \end{algorithmic}
\end{algorithm}

\section{Non-linear Extension Through Kernelization\label{sec:kernel}}
\begin{figure}[h]
    \centering
    \includegraphics[width=0.6\textwidth]{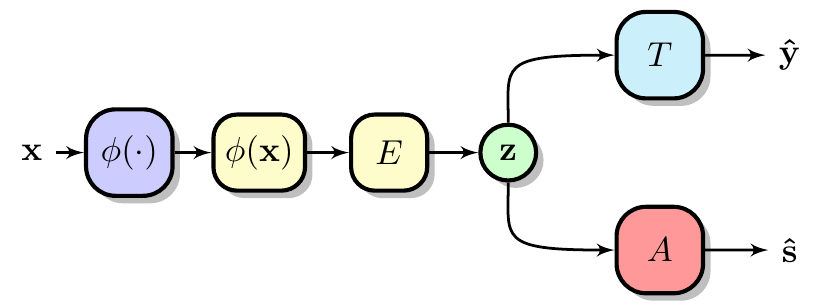}
    \caption{\textbf{Kernelized Adversarial Representation Learning} consists of four entities, a kernel $\phi_x(\cdot)$, an encoder $E$ that obtains a compact representation $\bm{z}$ of the mapped input data $\phi_x(\bm{x})$, a predictor $T$ that predicts a desired target attribute $\bm{y}$ and an adversary that seeks to extract a sensitive attribute $\bm{s}$, both from the embedding $\bm{z}$.\label{fig:kernel-arl}}
\end{figure}

We assume that $\x$ is non-linearly mapped to $\phi_x(\x)$ as illustrated in Figure~\ref{fig:kernel-arl}. From the representer theorem (see\cite{shawe2004kernel}), we note that $\bt_E$ can be expressed as $\bt_E=\bl \tilde{\bp}_x^T$. Consequently the embedded representation $\bm{z}$ can be computed as,
\[
\z = \bt_E \phi_x(\x) = \bl \tilde{\bp}_x^T \phi_x(\x) = \bl \bD^T[k_x(\x_1,\x),\cdots,k_x(\x_n,\x)]^T
\]

\subsection{Learning}
First, for a given fixed $\bt_E$, we find
\[
\J_y = \min_{\bW_y,\bb_y} \mse\,(\hat \y - \y).
\]
Note that the above optimization problem can be separated over $\bW_y$, $\bb_y$. Therefore, for a given $\bW_y$, we first minimize over $\bb_y$:
\begin{equation}
  \begin{aligned}
    &\min_{\bb_y}\E\Big\{\big\|\bW_y \bt_E \phi_x(\x)+\bb_y-\y\big\|^2\Big\}\\
    &=\min_{\bb_y}\frac{1}{n}\sum_{k=1}^n \big\|\bW_y \bt_E \phi_x(\x_k)+\bb_y-\y_k\big\|^2\\
    &=\frac{1}{n}\sum_{k=1}^n \big\|\bW_y \bt_E \phi_x(\x_k)+\bc-\y_k\big\|^2\nn
  \end{aligned}
\end{equation}
where the minimizer $\bc$ is,
\begin{equation}
  \label{c-kernel}
  \begin{aligned}
    \bc &= \frac{1}{n} \sum_{k=1}^n \Big(\y_k - \bW_y \bt_E \phi_x(\x_k)\Big)\\
    &= \frac{1}{n} \sum_{k=1}^n \y_k -  \bW_y \bt_E  \frac{1}{n} \sum_{k=1}^n \phi_x(\x_k)\\
    &= \E\big\{ \y \big\} - \bW_y \bt_E \,\E\big\{ \phi_x(\x) \big\}
\end{aligned}
\end{equation}
Let all the columns of $\bC$ be equal to $\bc$. Therefore we now have,
\begin{equation}
  \label{J_y-ker}
  \begin{aligned}
    &\min_{\bW_y,\bb_y} \mse\,(\hat \y - \y)\\
    &= \min_{\bW_y} \frac{1}{n}\big\| \bW_y \bt_E \bp_x + \bC-\bY \big\|_F^2\\
    &= \min_{\bW_y} \frac{1}{n}\big\| \bW_y \bt_E\tilde{\bp}_x -\tilde{\bY} \big\|_F^2\\
    &= \min_{\bW_y} \frac{1}{n}\big\| \tilde{\bp}_x^T \bt_E^T \bW_y^T - \tilde{\bY}^T \big\|_F^2\\
    &= \min_{\bW_y} \frac{1}{n}\big\| \bM \bW_y^T- P_\M \tilde {\bY}^T  \big\|_F^2 + \frac{1}{n}\big\| P_{\M^\perp} \tilde{\bY}^T \big\|_F^2\\
    &= \frac{1}{n}\big\| \underbrace{\bM \bM^\dagger}_{P_\M} P_\M \tilde {\bY}^T - P_\M \tilde {\bY}^T  \big\|_F^2 + \frac{1}{n}\big\| P_{\M^\perp} \tilde{\bY}^T \big\|_F^2\\
    &= \frac{1}{n}\big\| P_{\M^\perp} \tilde{\bY}^T \big\|_F^2\\
    &= \frac{1}{n}\big\| \tilde{\bY}^T \big\|_F^2-\frac{1}{n}\big\| P_\M \tilde{\bY}^T \big\|_F^2
  \end{aligned}
\end{equation}
where the third step is due to~(\ref{c-kernel}), $\bM=\tilde{\bp}_x ^T \bt_E^T$ and the fifth step is the orthogonal decomposition w.r.t. $\bM$. Using the same approach, we get
\begin{equation}
  \label{J_s-ker}
  \J_s = \frac{1}{n}\big\| \tilde{\bS}^T \big\|_F^2 - \frac{1}{n}\big\| P_\M \tilde{\bS}^T \big\|_F^2
\end{equation}

Finding optimal $\bt_E$ is equivalent to finding optimal $\bl_E$ (since $\bt_E = \bl_E \tilde{\bp}_x^T $) where we would have $\bM= \tilde{\bp}_x ^T \tilde{\bp}_x \bl^T_E= \tilde{\bK}_x \bl^T_E$. Now, assume that the columns of $\bL_x$ are orthogonal basis for the column space of $\tilde{\bK}_x$. As a result, for any $\bM$, there exists $\bG_E$ such that $\bL_x \bG_E =\bM$. In general, there is no bijection between $\bl_E$ and $\bG_E$ in the equality $\tilde{\bK}_x \bl^T_E = \bL_x \bG_E$. But, there is a bijection between $\bG_E$ and $\bl_E$ when constrained to $\bl_E$'s in which $\mathcal R(\bl^T_E)\subseteq {\N(\tilde{\bK}_x)}^\perp$. This restricted bijection is sufficient, since for any $\bl^T_E\in \N(\tilde{\bK}_x)$ we have $\bM=\bf0$. Once $\bG_E$ is determined, $\bl^T_E$ can be obtained as,
\[
\bl^T_E =  (\tilde{\bK}_x)^\dagger \bL_x \bG_E + \bl_0, \ \ \bl_0\subseteq \N(\tilde{\bK}_x)
\]
However, since
\[
\big\|\bl_E\big\|_F^2=\big\|\bl^T_E\big\|_F^2 = \big\|(\tilde{\bK}_x)^\dagger \bL_x \bG_E\big\|_F^2 + \big\|\bl_0\big\|_F^2,
\]
choosing $\bl_0 = \bf 0$ results in minimum $\big\|\bl_E\big\|_F$, which is favorable in terms of robustness to the noise. Similar to~(\ref{A-G}), we have $P_\M = \bL_x P_\G \bL_x^T$. If we assume that the rank of $P_\G$ is $i$, $J_\lambda(\bG_E)$ in~(\ref{eq-main3}) can be expressed as,
\begin{equation}
  \label{main-ker}
  J_\lambda(\bG_E) = \lambda \big\| \bL_x \bG_E \bG_E^T \bL_x^T \tilde{\bS}^T\big\|_F^2 - (1-\lambda) \big\| \bL_x \bG_E \bG_E^T \bL_x^T \tilde{\bY}^T\big\|_F^2\nn
\end{equation}
where $ P_\G = \bG_E \bG_E^T\nn $ for some orthogonal matrix $\bG_E\in \R^{d\times i}$. This resembles the optimization problem in (\ref{eq-main2}) and therefore have the same solution as Theorem~\ref{th3} with modified $\bB$ as,
\begin{equation}
  \label{B-ker}
  \bB = \bL_x^T \Big(\lambda \tilde{\bS}^T \tilde{\bS} -(1-\lambda)\tilde{\bY}^T \tilde{\bY} \Big) \bL_x
\end{equation}
Once $\bG_E$ is determined, $\bl_E$ can be computed as $\bG_E^T \bL_x^T (\tilde{\bK}_x^T)^\dagger$. Algorithm~\ref{alg1} summarizes our entire solution (replacing $\tilde{\bX}$ by $\tilde{\bK}_x^T$ in steps $3$ and $8$) if one wishes to consider the constrained optimization problem in~(\ref{eq-constrained}) instead of unconstrained Lagrangian version in~(\ref{eq-main}). It is worth of mentioning that the objective function $J_\lambda (\bG_E)$ is neither convex nor differentiable. The proof is exactly the same as Theorem~\ref{th3}.

\section{Proof of Lemma~4\label{sec:lemma4}}

\begin{lemma}
Let the columns of $\bL_x$ be the orthonormal basis for $\tilde{\bK}_x$ (in linear case $\tilde{\bK}_x = \tilde{\bX}^T \tilde{\bX}$). Further, assume that the columns of $\bV_s$ are the singular vectors corresponding to zero singular values of $\tilde{\bS}\bL_x$
and the  columns of $\bV_y$ are the singular vectors corresponding to non-zero singular values of $\tilde{\bY}\bL_x$.
Then, we have
\begin{align}
    \gamma_\min = &\min_{\bt_E}J_y(\bt_E)\nn\\
    = &\frac{1}{n}\big\|\tilde{\bY}^T\big\|_F^2-\frac{1}{n}{\|\tilde{\bY}\bL_x\|_F^2} \nn\\
        \gamma_\max = &\min_{\arg \max J_s(\bt_E)} J_y(\bt_E)\nn\\
        = &\frac{1}{n}\big\|\tilde{\bY}^T\big\|_F^2 - \frac{1}{n} \big\| \tilde{\bY} \bL_x \bV_s\big\|_F^2\nn\\
    \alpha_\min = &\max_{\arg \min J_y(\bt_E)} J_s(\bt_E)\nn\\
    = &\frac{1}{n}\big\|\tilde{\bS}^T\big\|_F^2 - \frac{1}{n} \big\| \tilde{\bS} \bL_x \bV_y\big\|_F^2\nn\\
    \alpha_\max = &\max_{\bt_E} J_s(\bt_E)\nn\\
    = &\frac{1}{n}\big\|\tilde{\bS}^T\big\|_F^2\nn.
\end{align}
\end{lemma}

\begin{proof}
Firstly, we recall from Section~\ref{sec:kernel} that instead of $\bl_E$, we consider $\bG_E$. These two matrices are related to each other as, $\tilde{\bK}_x\bl_E^T=\bL_x \bG_E=\bM$, where the columns of $\bL_x$ are the orthogonal basis for the column space of $\tilde{\bK}_x$. Therefore we can now express the projection on to $\M$ in terms of projection onto $\G$, i.e.,$P_\M = \bL_x P_\G \bL_x$. Using~(\ref{J_y-ker}), we get
\begin{equation}
  \label{eq-gamma-min}
  \begin{aligned}
    \gamma_\min &= \frac{1}{n}\big\|\tilde{\bY}^T\big\|_F^2-\frac{1}{n}\max_{\bt_E}{\big\|P_\M \tilde{\bY}^T\big\|_F^2}\\
    &= \frac{1}{n}\big\|\tilde{\bY}^T\big\|_F^2-\frac{1}{n}\max_{\bG_E}{\big\|\bL_x P_{\G_E} \bL_x^T \tilde{\bY}^T\big\|_F^2}\\
    &= \frac{1}{n}\big\|\tilde{\bY}^T\big\|_F^2 -\frac{1}{n}\max_i \Big\{\max_{\bG_E^T\bG_E=\bI_i}\tr\big[\bG_E^T \bL_x^T \tilde{\bY}^T \tilde{\bY}  \bL_x \bG_E \big]\Big\}\\
    &= \frac{1}{n}\big\|\tilde{\bY}^T\big\|_F^2 -\frac{1}{n}\tr\big[\bV_y^T \bL_x^T \tilde{\bY}^T \tilde{\bY}  \bL_x \bV_y \big]\\
    &= \frac{1}{n}\big\|\tilde{\bY}^T\big\|_F^2-\frac{1}{n}\sum_k \sigma_k^2\\
    &= \frac{1}{n}\big\|\tilde{\bY}^T\big\|_F^2-\frac{1}{n}\sum_{\sigma_k>0} \sigma_k^2\\
    &= \frac{1}{n}\big\|\tilde{\bY}^T\big\|_F^2-\frac{1}{n}{\big\|\tilde{\bY} \bL_x\big\|_F^2}
\end{aligned}
\end{equation}
where the fourth step is borrowed from trace optimization problems studied in~\cite{kokiopoulou2011trace} and $\sigma_k$'s are the singular values of $\tilde{\bY}\bL_x$.

In order to better interpret the bounds, we consider the one-dimensional case where $\x, \y, \in \R$. In this setting, the correlation coefficient (denoted by $\rho(\cdot, \cdot)$) between $\x$ and $\y$ is,
\begin{equation}
  \begin{aligned}
    \rho (\x,\y) &= \frac{\tilde{\bY } \tilde{\bX}^T}{\sqrt{\tilde{\bY}\tilde{\bY}^T \tilde{\bX}\tilde{\bX}^T}}\\
    &=\frac{\|\tilde{\bY}\bL_x\|_F}{\sigma_y}\\
    &=\sqrt{1-\frac{\gamma_\min}{\sigma_y^2}},
  \end{aligned}
\end{equation}
where $\sigma_y^2 = \| \tilde{\bY}\|^2_F/n$. As a result, the normalized MSE can be expressed as,
\begin{equation}
    \frac{\gamma_\min}{\sigma_y^2} = 1-\rho^2(\x,\y)
\end{equation}
Therefore, the lower bound of the target's MSE is independent of the encoder and is instead only related to the alignment between the subspaces spanned by the data and labels.

Next, we find an encoder which allows the target task to obtain its optimal loss, $\gamma_\min$, while seeking to minimize the leakage of sensitive attributes as much as possible. Thus, we constrain the domain of the encoder to $\{\arg\min J_y(\bt_E)\}$. Assume that the columns of the encoder $\bG_E$ is the concatenation of the columns of $\bV_y$ together with at least one singular vector corresponding to a zero singular value of $\tilde{\bY}\bL_x$. Therefore $\mathcal{V}_y \subseteq \mathcal{G}_E$ and consequently $\|\bL_x P_{\mathcal{V}_y} \bL_x^T {\bf U}\|^2_F \le \|\bL_x P_{\mathcal{G}} \bL_x^T {\bf U}\|^2_F$ for arbitrary matrix ${\bf U}$. As a result, $J_s(\bG_E) \ge J_s(\bV_y)$ and at the same time $J_y(\bG_E) = J_y(\bV_y)$. The latter can be observed from,
\begin{equation}
  \begin{aligned}
    \big\|\bL_x P_{\G_E} \bL_x^T \tilde{\bY}^T\big\|_F^2 &= \big\|\tilde{\bY}\bL_x P_{\G_E} \bL_x^T \big\|_F^2\\
    &= \big\|\tilde{\bY}\bL_x \bG_E \bG_E^T \bL_x^T \tilde{\bY}^T\big\|_F^2\\
    &= \big\|\tilde{\bY}\bL_x \bV_y \bV_y^T \bL_x^T \big\|_F^2\\
    &= \big\|\bL_x P_{\mathcal{V}_y} \bL_x^T \tilde{\bY}^T\big\|_F^2
  \end{aligned}
\end{equation}
We then have,
\begin{equation}
  \begin{aligned}
    \alpha_{\min} &= \frac{1}{n}\big\|\tilde{\bS}^T\big\|_F^2-\frac{1}{n}{\big\|\bL_x P_{\mathcal{V}_y} \bL_x^T\tilde{\bS}^T\big\|_F^2}\\
    &= \frac{1}{n}\big\|\tilde{\bS}^T\big\|_F^2 -\frac{1}{n}\tr\big[\bV_y^T \bL_x^T \tilde{\bS}^T \tilde{\bS} \bL_x \bV_y \big]\\
    &= \frac{1}{n}\big\|\tilde{\bS}^T\big\|_F^2 - \frac{1}{n} \big\| \tilde{\bS} \bL_x \bV_y\big\|_F^2
  \end{aligned}
\end{equation}

This bound can again be interpreted under the one-dimensional setting of $\x,\s \in \R$ as,
\begin{equation}
    \frac{\alpha_\min}{\sigma_s^2} = 1- \rho^2(\x,\s)
\end{equation}

On the other hand, $\alpha_\max$ turns out to be,
\begin{equation}
  \begin{aligned}
    \alpha_\max &= \frac{1}{n}\big\|\tilde{\bS}^T\big\|_F^2 \\
    &= \sigma_s^2
\end{aligned}
\end{equation}
which can be achieved via trivial choice of $\bG_E = 0$. However, we let the columns of $\bG_E$ be the singular vectors corresponding to all zero singular values of $\tilde{\bS}\bL_x$ in order to maximize $\big\|P_\M \tilde{\bY}^T\big\|_F$ and at the same time ensuring that $J_s(\bG_E)$ equal to $\alpha_\max$. As a result, we have
\begin{equation}
\gamma_\max = \frac{1}{n}\big\|\tilde{\bY}^T\big\|_F^2 -\frac{1}{n}\big\| \tilde{\bY} \bL_x \bV_s\big\|_F^2\nn.
\end{equation}
For the one dimensional case i.e., $\x,\y,\s \in \R$, we get $V_s =0$ and consequently,
\begin{equation}
    \gamma_\max = \sigma_y^2
\end{equation}
\end{proof}